\documentclass[11pt]{article}
\usepackage{mystyle}
\usepackage{fullpage}
\allowdisplaybreaks[4]

\usepackage{amsmath}
\usepackage{amssymb}
\usepackage{mathtools}
\usepackage{mathrsfs}
\usepackage{amsthm}

\usepackage{color}
\usepackage{algorithmic}
\usepackage{algorithm}
\usepackage{titletoc}
\usepackage{pdfpages}

\usepackage{hyperref,etoolbox}
\hypersetup{hidelinks}
\hypersetup{
colorlinks=true,
linkcolor=blue,
citecolor=blue
}

\def\##1\#{\begin{align}#1\end{align}}
\def\$#1\${\begin{align*}#1\end{align*}}

\usepackage{thm-restate}

\usepackage{xcolor}

\usepackage{enumitem}

\usepackage{bbding}
\usepackage{makecell}

\usepackage[utf8]{inputenc} 
\usepackage[T1]{fontenc}    
\usepackage{hyperref}       
\usepackage{url}            
\usepackage{booktabs}       
\usepackage{amsfonts}       
\usepackage{nicefrac}       
\usepackage{microtype}      
\usepackage{xcolor}         
\usepackage{graphicx}
\usepackage{titlesec}
\usepackage{subfigure}

\usepackage{titletoc}

\usepackage{booktabs} 
\usepackage{caption}
\theoremstyle{plain}

\usepackage[textsize=tiny]{todonotes}

\let\tilde\widetilde

\newcommand{\compilefullversion}{true}
\ifthenelse{\equal{\compilefullversion}{false}}{%
	\newcommand{\OnlyInFull}[1]{}
	\newcommand{\OnlyInShort}[1]{#1}
}{%
	\newcommand{\OnlyInFull}[1]{#1}%
	\newcommand{\OnlyInShort}[1]{}%
}%

\usepackage{color,xcolor}
\usepackage{colortbl}
\usepackage{bbm}
\usepackage{tcolorbox}

\definecolor{red1}{HTML}{f47983}
\definecolor{blue1}{HTML}{3eede7}
\definecolor{yellow1}{HTML}{f5dd6f}

\title{\LARGE A General Framework for Sequential Decision-Making under Adaptivity Constraints}
\author{Nuoya Xiong\thanks{IIIS, Tsinghua University. Email: \texttt{xiongny20@mails.tsinghua.edu.cn}.} \qquad\quad    Zhaoran Wang\thanks{Northwestern University. Email: \texttt{zhaoranwang@gmail.com}.}
\qquad \quad Zhuoran Yang\thanks{Yale University. Email: \texttt{zhuoran.yang@yale.edu}.} }
\date{}
\begin{document}

\maketitle


\begin{abstract} 
     We take the first step in studying general sequential decision-making under two adaptivity constraints:  rare policy switch and batch learning.  First, we provide a general class called the Eluder Condition class, which includes a wide range of  reinforcement learning classes. 
    Then, for the rare policy switch constraint, we provide a generic algorithm to achieve a $\widetilde{\cO}(\log K) $ switching cost with a $\widetilde{\cO}(\sqrt{K})$ regret on the EC class. For the batch learning constraint, we provide an algorithm that provides a $\widetilde{\cO}(\sqrt{K}+K/B)$ regret with the number of batches $B.$
    This paper is the first work considering rare policy switch and  batch learning  under general function classes, which covers nearly all the models studied in the previous works such as  tabular MDP \citep{baiyu2019lowswitchcosttabular,zhang2020lowswitchcosttabular}, linear MDP \citep{quanquangu2021linearlowcost,gao2021lowcostlinear}, low eluder dimension MDP \citep{kong2021eluderlowcost,zhuoranyang2022lowcosteluder}, generalized linear function approximation \citep{qiao2023logarithmic}, and also some new classes such as the low $D_\Delta$-type Bellman eluder dimension problem, linear mixture MDP, kernelized nonlinear regulator and undercomplete partially observed Markov decision process (POMDP). 
\end{abstract}

\section{Introduction}

Reinforcement Learning (RL) provides a systematic  framework for  solving large-scale sequential decision-making problems and has demonstrated striking empirical successes across various domains \cite{li2017deep}, including  games \citep{silver2016mastering,vinyals2019grandmaster},  robotic control \citep{akkaya2019solving},  healthcare   \citep{yu2021reinforcement},  hardware device placement \citep{mirhoseini2017device}, recommender systems \citep{zou2020pseudo}, and so on.

In the online setting, an RL algorithm iteratively finds the optimal policy of the sequential decision-making problem by (i) deploying the current policy to gather data and (ii) using the collected data to learn an improved policy.  
Most of  provably sample efficient algorithms  in the existing literature consider  an ideal setting where policy updates can be \emph{fully adaptive}, i.e., the policy can be updated after each episode, using the data sampled from this newly finished  episode. 
From a practical perspective, however, updating the policy after each episode can be unrealistic,  
especially  when  computation  resources are limited, or  the cost of policy switching is prohibitively high, or the data is not fully serial. 
For example, in recommender systems, it is unrealistic to change the policy after each instantaneous data such as a click of one of the customers. 
Moreover, the customers might not come in a serial manner -- it is possible that multiple customers arrive at the same time and we need to 
make simultaneous decisions. 
Similarly, when the RL  algorithms are deployed  on the large-scale hardwares,  changing a policy may need recompiling the code or changing the physical placement for devices, incurring considerable switching costs. 
Thus, when it comes to designing RL algorithms in  these scenarios, in addition to achieving sample efficiency, 
we also aim to  reduce  or limit the number of policy switches.

Such an additional restriction is known as 
the \textit{adaptivity constraints} \citep{quanquangu2021linearlowcost}. 
There are two common types of adaptivity constraints: the rare policy switch constraint \citep{perchet2016batched,gu2021batched} and the batch learning constraint~\citep{abbasi2011improved}. With the rare policy switch constraint, the agent adaptively decides when to update the policy during the course of the online reinforcement learning, 
and the goal is 
to achieve the sample efficiency that is comparable to the fully adaptive setting, while minimizing the number of policy switches. 
With the batch learning constraint, the total number of  batches $B$ is pre-determined, and the agent has to follow the same policy within each batch. 
In other words, the number of policy switches is limited by the number of batches. 
In addition to designing the policies used in each batch, the agent additionally needs to decide 
how to split the total $K$ episodes in to $B$ batches before interacting with the environment.

Moreover, in many real applications of RL such as recommender systems, the state space can be extremely large or even infinite \citep{chen2019top}. 
Function approximation is an effective tool for handling 
such a  challenge 
and has been extensively studied in the literature under the fully adaptive setting  \citep{jiang@2017,sun2019model,foster2021statistical, jin2021bellman,zhong2022gec, chen2022abc}. 
Accordingly, a few previous works provide provably sample-efficient RL algorithms under adaptivity constraints for MDPs with  linear and generalized linear structures \cite{gao2021lowcostlinear, quanquangu2021linearlowcost,qiao2023logarithmic} and 
low-eluder-dimension MDPs \citep{kong2021eluderlowcost, zhuoranyang2022lowcosteluder}. 
However, it remains open when considering more general classes such as Bellman Eluder dimension \citep{jin2021bellman,qiao2023logarithmic}. This motivates us to consider the following question:
 
 \begin{center} {\bf Can we design  sample-efficient RLs algorithm under adaptivity constraints in the context of general function approximation?}
 \end{center}

In this work, we establish the first algorithmic framework under adaptivity constraints for a general function  class named \textbf{E}luder-\textbf{C}ondition (EC) class. 
Our framework  applies to both single-agent MDP and one player in a zero-sum Markov game. 
Besides,  EC class contains many popular MDP and Markov game models studied in the previous literature.  
Some examples contained in our framework and the comparison of previous works are shown in Table~\ref{table:comparison}.  
We also provide some additional examples in \S\ref{sec:Examples} and \S \OnlyInFull{\ref{appendix:additional examples}}\OnlyInShort{F}.
 
 For the rare policy switch problem, motivated by the optimistic  algorithms such as  GOLF \citep{jin2021bellman} and OPERA \citep{chen2022abc}, our algorithm  constructs an optimal confidence set and computes the optimal policy via optimistic planning based on the confidence set.  
 Rather than updating the confidence set at each episode, we use a more delicate strategy to reduce the number of policy switches. 
As we employ optimistic planning, switching a policy essentially means that we update the confidence set that contains the true hypothesis. 
To reduce the number of policy switches, we update the confidence set only when the update provide considerable improvement in terms of the estimation. 
In particular, 
in each episode, we 
first estimate the  
improvement provided by the new confidence set, and  then only  decide to update the confidence set and optimal policy when the estimated improvement exceeds a certain threshold. 
Our lazy switching strategy  can reduce the number of policy switches from $\cO(K)$ to $\cO(\mbox{poly}(\log K))$, leading to an exponential improvement in terms of policy switches. We also refer to the number of policy switches as \textit{the switching cost}.  Meanwhile, for the batch learning problem, we use a fixed uniform grid that divides $K$ episodes into $B$ batches. While the uniform grid is an intuitive and common  choice \citep{quanquangu2021linearlowcost, han2020sequential}, analyzing the regret of this approach under general function classes requires new techniques.
 Our work takes the first step in studying the rare policy switch problem and the batch learning problem with general function approximation.
 \begin{table}\footnotesize\centering
 \caption{Comparison of previous representative work for adaptivity constraints and our works. For function classes with low eluder dimension, \cite{kong2021eluderlowcost} needs a strong value-closeness assumption, while we do not need that assumption.}
\begin{tabular}{ccccc}

\midrule[1.5pt]
 &  \makecell{\cite{quanquangu2021linearlowcost}} &  \makecell{\cite{kong2021eluderlowcost}}& \cite{qiao2023logarithmic}  & Ours \\ \hline
\makecell{Tabular MDP}    & \tiny\Checkmark & \tiny\Checkmark & \tiny\Checkmark & \tiny\Checkmark\\ \hline
\makecell{Linear MDP\vspace{-0.3em}\\ \tiny\citep{jin2020provably}}      & \tiny\Checkmark  & \tiny\Checkmark  & \tiny\Checkmark  & \tiny\Checkmark  \\ \hline
\makecell{Low Eluder Dimension\vspace{-0.3em}\\\tiny\citep{russo2013eluder}} &   \tiny\XSolidBrush   &   \makecell{\tiny\Checkmark}   &   \tiny\XSolidBrush   &     \makecell{\tiny\Checkmark} \\ \hline
\makecell{Low Inherent Bellman Error\vspace{-0.3em}\\\tiny\citep{zanette2020learning}} & \tiny\XSolidBrush &   \tiny\XSolidBrush   &  \tiny\Checkmark    &   \tiny\XSolidBrush   \\ \hline
\makecell{Low $D_\Delta$-type BE Dimension\vspace{-0.3em}\\\tiny\citep{jin2021bellman}} & \tiny\XSolidBrush  & \tiny\XSolidBrush  & \tiny\XSolidBrush  & \tiny\Checkmark \\ \hline
\makecell{Linear Mixture MDP \vspace{-0.3em}\\\tiny\citep{ayoub2020model}}  &   \tiny\XSolidBrush   &   \tiny\XSolidBrush   &  \tiny\XSolidBrush    &   \tiny\Checkmark   \\ \hline
\makecell{Kernelized Nonlinear Regulator\vspace{-0.3em}\\\tiny\citep{kakade2020information}}     &    \tiny\XSolidBrush  &   \tiny\XSolidBrush   &   \tiny\XSolidBrush   &   \tiny\Checkmark   \\ \hline
\makecell{SAIL Condition\vspace{-0.3em}\\\tiny\citep{liu2022optimistic}}   &  \tiny\XSolidBrush    &  \tiny\XSolidBrush    &    \tiny\XSolidBrush  &     \tiny\Checkmark \\ \hline
\makecell{Undercomplete POMDP   \vspace{-0.3em}\\\tiny\citep{liu2022partially}} &    \tiny\XSolidBrush &   \tiny\XSolidBrush   &  \tiny\XSolidBrush    &    \tiny\Checkmark  \\ \hline
\makecell{Zero-Sum Markov Games with\\Low Minimax BE Dimension\vspace{-0.3em}\\\tiny\citep{huang2021markovgamegeneral}}   &  \tiny\XSolidBrush    &  \tiny\XSolidBrush    &  \tiny\XSolidBrush    &   \tiny\Checkmark   \\ \bottomrule[1.5pt]
\end{tabular}
\vspace{0.6em}
\label{table:comparison}
\end{table}

In summary, we make the following contributions:
  
  $\bullet$   We provide a general function classes called the $\ell_2$-Eluder Condition (EC) class and the $\ell_1$-EC class, and then show that the EC class contains a wide range of previous RL models with general function approximation, such as low $D_\Delta$-type Bellman eluder dimension model, linear mixture MDP, KNR and  Generalized Linear Bellman Complete MDP.  

   $\bullet$   We develop a generic algorithm $\ell_2$-EC-Rare Switch (RS) for the $\ell_2$-type EC class. The algorithm uses optimistic estimation to achieve a  $\widetilde{\cO}(H\sqrt{dK})$ regret,  updates the confidence set, and changes the policy by a delicate strategy to achieve a $\widetilde{\cO}(dH \log K)$ switching cost, where $d$ is the parameter in the EC class and $\widetilde{\cO}$ contains the logarithmic term except for $\log K$.   In Appendix \OnlyInFull{\ref{appendix:l1 EC class}}\OnlyInShort{C}, we also provide $\ell_1$-EC-RS algorithm for the $\ell_1$-type EC class. 
    We apply our results to some specific examples, showing that our method is sample-efficient with a low switching cost. 

    $\bullet$  For batch learning problems, we also develop an intuitive 
 and generic algorithm $\ell_2$-EC-Batch that achieves a $\widetilde{\cO}(\sqrt{dK}+dK/B)$ regret, where $B$ is the number of batches. Our regret is comparable to the existing works for batch learning in the linear MDP \citep{quanquangu2021linearlowcost} and also matches their lower bound.

{\noindent \bf Related Works.} Our paper is closely related to the prior research on RL with general function approximation and RL with adaptivity constraints. A comprehensive summary of the related literature can be found in \S \ref{appendix: related work}.

\section{Preliminaries}\label{sec:background}

\paragraph{Episodic  MDP}

A finite-horizon, episodic Markov decision process (MDP) is represented by a tuple $(\cS, \cA, H, \PP, r)$, where $\cS$ and $\cA$ denote the state space and action space; $H$ is the length of each episode, 
$\PP = \{\PP_h\}_{h \in [H]}$ is the transition kernel, where $\PP_h(s_{h+1} \mid s_h,a_h)$ represents the probability to arrive state $s_{h+1}$ when taking action $a_h$ on state $s_h$ at step $h$; $r = \{r_h(s,a)\}_{h \in [H]}$ denotes the deterministic reward function after taking action $a$ at state $s$ and step $h$.  We assume $\sum_{h=1}^H r_h(s_h,a_h) \in [0,1]$ for all possible sequences $\{s_h,a_h\}_{h=1}^H $.
A deterministic Markov policy $\pi = \{\pi_h\}_{h \in [H]}$ is a set of $H$ functions, where $\pi_h: \cS \rightarrow \Delta_{\cA}$ is a mapping from state to an action. 
For any policy $\pi$, its action value function $Q^{\pi}_h(s,a)$ and state value function $V_h^\pi(s)$ are defined as 
\$
Q^\pi_h(s,a) &= \EE_{\pi} \Bigg[\sum_{h' = h}^H r_{h'}(s_{h'}, \pi_{h'}(s_{h'})) \Bigg|s_{h}=s, a_{h}=a \Bigg],\\
V^\pi_h(s) &= \EE_{\pi} \Bigg[\sum_{h' = h}^H r_{h'}(s_{h'}, \pi_{h'}(s_{h'})) \Bigg|s_{h}=s \Bigg].
\$
To simplify the presentation, without loss of generality, we assume  the initial state is fixed at $s_1$. The optimal policy $\pi^*$ maximizes the value function, i.e., $\pi^* = \argmax_{\pi \in \Pi} V_1^{\pi}(s_1)$. We also denote $V^* = V^{\pi^*}$ and $Q^* = Q^{\pi^*}$. 
Note that the function $Q^*$ is the unique solution to  the Bellman equations $Q_h^*(s,a)=\cT_hQ_{h+1}^*(s,a) $, where the Bellman operator $\cT$ is defined by
\begin{align}
(\cT_hQ_{h+1})(s,a) = r_h(s,a) + \EE_{s'\sim \PP_h(s'\mid s,a)}\max_{a \in \cA}Q_ {h+1} (s,a).\label{eq:bellman operator for MDP}
\end{align}
We also study the low switching-cost problems under zero-sum Markov Games, and we put the definitions,  learning objective, algorithm and results into \S\ref{def of zero-sum MG}.

\paragraph{Function Approximation}
Generally speaking, under the function approximation setting, we can access to a hypothesis class $\cF$ which captures the key feature of the value functions (in the model-free setting) or the transition kernels and the reward functions (in the model-based setting) of the RL problem. In specific, 
let $M$ denote the MDP instance, which is clear from the context. We assume that we have access to a hypothesis class $\cF = \cF_1\times \cdots \times \cF_H$, where a  hypothesis function $f = \{ f_1, \ldots,  f_H \}  \in \cF $ 
either represents an  action-value function $Q_f = \{ Q_{h,f}\}_{h\in [H]}$ in the model-free setting, or the environment model of MDP $M_f = \{ \PP_{h,f} ,  r_{h,f} \}_{h\in [H]}$ in the model-based setting. 
Moreover, for any $f \in \cF$, let $\pi_f$ denote the optimal policy corresponding to the hypothesis $f$. 
That is, 
under the model-free setting, $\pi_f$ is the greedy policy with respect to $Q_f$, i.e.,   $\pi_{h,f}(s) = \arg\max_{a \in \cA}Q_{h,f}(s,a)$. 
Moreover, given $Q_f$ and $\pi_f$, we define state-value function $V_f$ by letting 
$V_{h,f}(s) =  \EE_{a \sim \pi_{h,f}(s)}[Q_{h,f}(s,a)]$ in the MDP. 
Furthermore, under the model-based setting, let $M_f=(\PP_f, r_f)$ be the transition kernel and reward function associated with the hypothesis $f$. 



Similar to the previous works  \citep{chen2022abc}, we impose the  following realizability assumption to make sure the true MDP or MG model $M$  is captured by  the hypothesis class $\cF$.
\begin{assumption}[Realizability]\label{assum:realizability}
    A hypothesis class $\cF$ satisfies the realizability condition  if there exists a hypothesis function $f^* \in \cF$ such that  $Q_{h,f^*} = Q_h^*, V_{h,f^*} = V_h^*$  for all $h\in [H]$.
\end{assumption}


\paragraph{Learning Goal}
In this paper, we aim to design online reinforcement learning algorithms for the  rare policy switches problem  and the
batched learning problem.  Assume the agent executes the policy $\pi^k$ in the $k$-th episode for all  $k \in [K]$, the regret of the agent is defined as
    \#\label{eq:define_regret}
    R(K) = \sum_{t= 1}^K  (V^*_1(s_1) - V^{\pi^k}_1(s_1) ) .
    \#

 The switching cost is the number of policy switches during the interactive process. Assume the agent uses the policy $\pi^t$ in $t$-th episode, the switch cost at $T$-th episodes is:
\$
    N_{\mbox{switch}}(K) = \sum_{k=1}^K \II\{\pi^k \neq \pi^{k+1}\}.
\$
In this paper, for the rare policy switch problem, we aim to achieve a logarithmic switching cost and maintain a $\widetilde{\cO}(\sqrt{K})$ regret.
In other words, 
we aim to design an algorithm such that $ R(K) = \tilde \cO(\sqrt{K})$ while $N_\mathrm{switch}(K) = \mathrm{poly}\log(K)$, where we ignore problem dependent quantities and $\tilde \cO(\cdot ) $ omits logarithmic terms.

For the batch learning problem, let $B \in [K] $ be a fixed integer. The agent of the MDP or the max-player in the zero-sum MG \emph{pre-determines} a grid $1= k_1<k_2<\cdots<k_{B+1}=K+1$ with $B+1$ points that split the $K$ episodes into $B$ batches 
$\{ k_1, \ldots,  k_2 -1 \} , \ldots, \{ k_2, \ldots, k_3 -1 \} , \ldots, \{ k_{B-1}, \ldots 
k_{B} -1 \} ,  \{ k_B, \ldots, K\}$. 
In an MDP, the agent can only execute the same policy within a batch and change the policy only at the end of a batch. 
In a zero-sum MG, batch learning requires that 
the max-player can only change her policies at the end of each batch. 
Meanwhile, the min-player is free to change the policy after each episode. 
Similarly, the agent (max-player) aims  to minimize the regret in \eqref{eq:define_regret} by (a) selecting the batching grid at the beginning of the algorithm and (b) designing the $B$ policies that are executed in each batch. 
Furthermore, in the case of $B= K$, the problem is reduced to a standard online reinforcement learning  problem.  

The difference between the 
rare policy switch setting and batch learning setting is that, 
in the former case, the algorithm can adaptively  decide when to switch the policy based on the data, whereas in the latter case, the episodes where the agent adopts a new policy are deterministically decided before the first episode.  
In other words, 
in reinforcement learning with rare policy switches, we are confident to achieve a sublinear $\tilde \cO(\sqrt{K})$ regret, e.g., using an online reinforcement learning algorithm that switches the policy after each episode. 
The goal is to attain the desired regret with a small  number of policy switches. 
In contrast, 
in the batch learning setting, with $B$ fixed, we aim to minimize the regret, under the restriction that the number of policy switches is no more than $B$.

\section{Eluder-Condition Class} \label{sec:ET}

To handle the RL problems with adaptivity constraints, 
we propose a general class called Eluder-Condition (EC) class, which has a stronger eluder assumption and thus helps us to control the adaptivity constraints.
 There are two types of EC class: $\ell_2$-EC class and $\ell_1$-EC class. 
 We mainly discuss the $\ell_2$-EC class in the main text, and introduce the $\ell_1$-EC class in \S \OnlyInFull{\ref{appendix:l1 EC class}}\OnlyInShort{C}. 
 We first consider the function class with low $D_\Delta$-type BE dimension \citep{jin2021bellman} as a primary example to show our stronger eluder assumption. Define the Bellman residual $\cE_h(f)(s_h,a_h) = (f_h-\cT(f_{h+1}))(s_h,a_h)$ for all $h \in [H].$ In the eluder argument (Lemma 17) of \cite{jin2021bellman}, it is proven  that for any sequence $\{f^k\}_{k=1}^K$, if the Bellman error of $f^k$ and historical data $\{s_h^i,a_h^i\}_{i=1}^{k-1}$ satisfy 
 $
     \sum_{i=1}^{k-1} \left(\cE(f_h^k)(s_h^i,a_h^i)\right)^2 \le \beta,
 $
 then the in-sample error can be bounded by $\widetilde{\cO}(\sqrt{K})$, i.e. 
 \begin{align}\label{intuition: i,i}
     \sum_{i=1}^{k} \left|\cE(f_h^i)(s_h^i,a_h^i)\right| \le \cO(\sqrt{d\beta k}), \ \forall k \in [K].
 \end{align}
 However, \eqref{intuition: i,i} is not enough to control the adaptivity constraints such as the switching cost. 
 Instead, we find that a slightly stronger assumption in  \eqref{intuition:stronger} below helps us reduce the switching cost.
\begin{align}\label{intuition:stronger}
     \sum_{i=1}^{k} \left(\cE(f_h^i)(s_h^i,a_h^i)\right)^2 \le \cO(d\beta \log k), \ \forall k \in [K].
 \end{align}
 It is easy to show that \eqref{intuition:stronger} is slightly stronger than  \eqref{intuition: i,i} by Cauchy's inequality. However, this stronger assumption  enables  help us to achieve a low switching cost through some additional analyses.
Moreover, in Section \ref{sec:Examples}, we show that  \eqref{intuition:stronger} also holds for a wide range of tractable RL problems studied in the previous work such as linear mixture MDP, $D_\Delta$-type BE dimension and KNR. 
Now we provide the formal definition of the $\ell_2$-EC class.
To provide a unified treatment for both MDP and MG, we let $\{ \zeta_h , \eta_h\}_{h\in [H]} $ be subsets of the trajectory.  
In particular, we let $\eta_h = \{ s_h, a_h\}$ and $\zeta _h = \{s_{h+1} \}$ in a single-agent MDP, and let $\eta_{h} = \{ s_h, a_h, b_h \}$ and $\zeta_h = \{ s_{h+1}\}$ in a two-player zero-sum MG.

\begin{definition}[$\ell_2$-type EC Class]
    Given a   MDP or MG instance $M$, let $\cF$ and $\cG$ be  two hypothesis function classes   satisfying the realizability Assumption~\ref{assum:realizability} with  $\mathcal{F}\subset \mathcal{G}$. 
    For any $h\in [H]$ and $f' \in \cF$, let
    $\ell_{h,f'}(\zeta_h, \eta_h, f,g)$ be a vector-valued and bounded   loss function which serves as a proxy of the Bellman error at step $h$, where $f, f'\in \cF, g \in \cG$, and $\zeta_h, \eta_h$ are subsets of trajectory defined above.
    Moreover, we assume that 
    $\Vert\ell_{h,f'}(\zeta_h, \eta_h, f,g)\Vert_2$ is upper bounded by a constant $R$ for all $h$, $(f', f,g)$, and $(\zeta_h, \eta_h)$. 
    For parameters $d$ and $\kappa$, we say that $(M, \mathcal{F}, \mathcal{G}, \ell,  d,\kappa)$ is a $\ell_2$-type  EC class if the following two conditions hold for any $\beta\ge R^2$ and $h \in [H]$:


    (i). ($\ell_2$-type Eluder Condition) 
    For any $K$ hypotheses $f^1, \ldots, f^K \in \cF$, 
    if
    \begin{align}\sum_{i=1}^{k-1}\Big\Vert\mathbb{E}_{\zeta_h}\Big[\ell_{h,f^i}(\zeta_h, \eta_h^i, f^k,f^k)\Big] \Big\Vert_2^2 \le \beta\label{eq:l2 condition precondition}\end{align} holds for any $k \in [K]$, then  we have
    \begin{align}
        \sum_{i=1}^{k} \Big\Vert\mathbb{E}_{\zeta_h}\Big[\ell_{h,f^i}(\zeta_h, \eta_h^i, f^i, f^i)\Big]\Big\Vert_2^2 \le \cO(d\beta\log k), \ \ \forall \ k \in [K], \label{eq:l2cond}
    \end{align}
    where we consider $R$ as a constant and ignore it in $\cO(\cdot)$.
    

    (ii). ($\kappa$-Dominance) There exists a parameter $\kappa$ such that, for any $k \in [K]$, with probability at least $1-\delta$,
\begin{align} 
        \sum_{i=1}^k \big (V_{1,f^i}(s_1)-V^{\pi_i}(s_1) \big )
        \le \kappa \cdot \left(\sum_{h=1}^H \sum_{i=1}^{k} \EE_{\eta_h \sim \pi_i} \Bigl [ \Big\Vert\mathbb{E}_{\zeta_h}[ \ell_{h,f^i}(\zeta_h, \eta_h, f^i,f^i)] \Big\Vert_2 \Bigr] \right)\label{def:loss dominance}. 
    \end{align}
\end{definition}


In this definition, 
the $\kappa$-dominance property \eqref{def:loss dominance} shows that the final regret is upper bounded by the cumulative expectation of in-sample loss, which is standard in many previous works~\citep{du2021bilinear,chen2022abc}. The $\ell_2$-type eluder condition is a generalized version of  \eqref{intuition:stronger}. 
Indeed, when we choose $\zeta_h = \{s_{h+1}\}, \eta_h = \{s_h,a_h\}$, and $\ell_{h,f'}(\zeta_h, \eta_h, f, g) = Q_{h,g}(s_h,a_h) - r(s_h,a_h) - V_{h+1,f}(s_{h+1})$, the $\ell_2$-type condition in  \eqref{eq:l2cond} can be regarded as the condition involving the Bellman error, as shown in  \eqref{intuition: i,i}. 
Intuitively, the term $\sum_{i=1}^{k-1} \|\EE_{\zeta_h}[\ell_{h,f^i}(\zeta_h, \eta_h^i,f^k,f^k)]\|^2 $ in Eq.\eqref{eq:l2 condition precondition} represents the discrepancy between the function $f^k$ and the previous data. This term can be regarded as the estimation error after $k-1$ episodes. The term $\sum_{i=1}^{k} \|\mathbb{E}_{\zeta_h}[\ell_{h,f^i}(\zeta_h, \eta_h^i, f^i, f^i)]\|_2^2$ in Eq.\eqref{eq:l2cond} represents the discrepancy between $f^i$ and the data of $i$-th episode, which serves as an upper bound of the regret incurred in the first $k$ episodes due to the $\kappa$-dominance condition Eq.\eqref{def:loss dominance}.
Hence EC class connects these two terms which has the following implication: The regret of an optimistic algorithm is small as long as it generates a sequence of functions $\{f^k\}_{k \in [K]}$ such that the estimation error of $f^k$ on the data given by the previous $k-1$ episodes is small. The parameter $d$
 quantifies the hardness of achieving low regret via a small estimation error.
From the previous discussion, it is easy to show that our assumption is stricter than the previous works, and this stricter assumption can help us to reduce the switching cost by some additional original analyses. 
As we will show later in Section \ref{sec:Examples}, it is satisfied by many previous important models like $D_\Delta$-type BE dimension \citep{jin2021bellman}, which includes low eluder dimension~\citep{kong2021eluderlowcost, zhuoranyang2022lowcosteluder} and linear MDP~\citep{gao2021lowcostlinear, quanquangu2021linearlowcost}.

Moreover, in the $\ell_2$-type EC class, we consider the decomposable loss function \citep{chen2022abc}. The decomposable property generalizes one of the properties of Bellman error and implies the completeness assumption in previous work \citep{jin2021bellman}.
\begin{definition}[Decomposable Loss Function (DLF) \citep{chen2022abc}]\label{def:DLF}
 The loss function $\ell_{h,f'}(\zeta_h, \eta_h, f, g)$ is decomposable if there exists an operator $\mathcal{T}:\mathcal{F}\to \mathcal{G}$, such that 
    \begin{align}
        \ell_{h,f'} (\zeta_h, \eta_h,  f, g) - \mathbb{E}_{\zeta_h}\Big[\ell_{h,f'}(\zeta_h, \eta_h, f, g)\Big] = \ell_{h,f'}(\zeta_h, \eta_h, f,\mathcal{T}(f)).\label{eq:dlf}
    \end{align}
    Also, the operator $\cT$ satisfies that $\mathcal{T}(f^*) = f^*$.
\end{definition} 
The decomposable property claims that for any $f \in \cF, g \in \cG$, there exists a function $\cT(f) \in \cG$ that is independent of  $g$ satisfying  \eqref{eq:dlf}, which can be regarded as a generalized completeness assumption. 
For example, in the function classes with low $D_\Delta$-type BE dimension for single-agent MDP, the operator $\cT$ is selected as the Bellman operator for single-agent MDP, which is given in  \eqref{eq:bellman operator for MDP}.
In this case, we can choose when we choose $\zeta_h = \{s_{h+1}\}, \eta_h = \{s_h,a_h\}$, and $\ell_{h,f'}(\zeta_h, \eta_h, f, g) = Q_{h,g}(s_h,a_h) - r(s_h,a_h) - V_{h+1,f}(s_{h+1})$, then 
we have 
\begin{align*}
& \ell_{h,f'}(\zeta_h, \eta_h,f,g) - \EE_{\zeta_h}\big[\ell_{h,f'}(\zeta_h, \eta_h,f,g)\big] \notag \\
& \qquad   =(\cT_hV_{h+1}(s_h,a_h)-r(s_h,a_h) - V_{h+1,f}(s_{h+1}))  =\ell_{h,f'}(\zeta_h, \eta_h, f, \cT(f)) , 
\end{align*}  
where $\cT$ is the Bellman operator. 
See Section \ref{sec:Examples} for details. 
For some other examples like linear mixture MDP and KNR, the operator $\cT$ is chosen as the optimal operator $\cT(f) = f^*$.

In recent years, the $\ell_1$-eluder argument is proposed in \cite{liu2022partially} and followed by \cite{liu2022optimistic} to provide another way for the sample-efficient algorithm of POMDP. 
In  \S \OnlyInFull{\ref{appendix:l1 EC class}}\OnlyInShort{C}, we also provide a similar EC class named $\ell_1$-type EC class. 
Compared to the $\ell_2$-EC class,  $\ell_1$-EC class replaces the square sum in  the $\ell_2$-type EC property (Eq. \eqref{eq:l2cond}) by a standard  sum. By considering a particular model-based loss function, the $\ell_1$-type EC class can reduce to the assumption in \cite{liu2022optimistic}. 
We provide a sample-efficient algorithm for $\ell_1$-EC class with low switching cost
in \S \OnlyInFull{\ref{appendix:l1 EC class}}\OnlyInShort{B} and a batch learning algorithm in Appendix \ref{appendix:batch l1}.

\section{Rare Policy Switch Problem} \label{sec:alg}


In this section, we propose an algorithm for the $\ell_2$-type EC class that achieves a low switching cost. 
Our algorithm extends the optimistic-based exploration algorithm \citep{jin2021bellman,chen2022abc} with a lazy policy switches strategy.
The optimistic-based exploration algorithm calculates a confidence set using historical data and performs optimistic planning within this set to determine the optimal model $f^k$ and policy $\pi^k$ at each episode $k$.
 Unlike the previous algorithm, we choose to update the confidence set only when a specific condition holds. This modification helps reduce the frequency of policy switches and lowers the associated cost.

\begin{algorithm}[t]
    \begin{algorithmic}[1]
        
	\caption{$\ell_2$-EC-RS}
	\label{alg:ET-Rare switch}
	\STATE {\textbf{Initialize:}} $D_1,D_2,\cdots,D_H=\emptyset,\mathscr{B}_1 = \cF$.
	
	\FOR{$k=1,2,\cdots,K$}
            
	    \STATE Compute $\pi^k = \pi_{f^k}$, where $f^k = \arg\max_{f \in \mathscr{B}^{k-1}}V_{f}^{\pi_f}(s_1)$.\label{line:oracle}
	    \STATE Execute policy $\pi^k$ to collect  the trajectory, update $D_h = D_h\cup\{\zeta_h^k, \eta_h^k\}, \forall h\in [H]$.

            
	    \IF {$L_h^{1:k}(D^{1:k}_h, f^k, f^k)-\inf_{g \in \cG}L_h^{1:k}(D^{1:k}_h,f^k, g)\ge 5\beta$ for some $h \in [H]$} \label{alg:l2beginif}
     
        \STATE Update \begin{align*}
	        \mathscr{B}^{k}=\left\{f \in \cF: L_h^{1:k}(D^{1:k}_h, f, f)-\inf_{g \in \cG}L_h^{1:k}(D^{1:k}_h,f, g)\le \beta, \forall h \in [H] \right\}.
	    \end{align*}\label{line:l2confidenceset}
            \ELSE \STATE $\mathscr{B}^{k}=\mathscr{B}^{k-1}$.\label{line:updateend}    \ENDIF\label{alg:l2endif}
	\ENDFOR
 \end{algorithmic}
\end{algorithm}

\vspace{4pt}
\noindent{\bf Optimistic Exploration with Low-Switching Cost}. 
At episode $k$, the agent first computes the optimal policy $\pi^k  = \pi_{f^k}$, where $f^k$ is the optimal model in confidence set $\rB_k$ and $\pi_{f}$ is the greedy policy with respect to $Q_f.$ Then it executes the policy $\pi^k$ (or $\upsilon^k$ for zero-sum MG) and collects the data $D_h$ for each step $h \in [H]$. Line~\ref{alg:l2beginif} - \ref{alg:l2endif} compute the optimistic confidence set $\rB^{k+1}$ for next episode $k+1$. Define the loss function 
\$
    L^{a:b}_{h}(D_h^{a:b},f,g) =\sum_{i=a}^{b} \Vert \ell_{h,f^i}(\zeta_h^i, \eta_h^i, f, g)\Vert_2^2,
\$
and calculate the confidence sets in Line~\ref{line:l2confidenceset} based on history data. 
Unlike the previous algorithm, our algorithm provides a novel policy switching condition in Line~\ref{alg:l2beginif}, which is the following inequality:
\begin{align}\label{eq:updating rule}
    L_{h}^{1:k}(D_h^{1:k}, f^k, f^k)-\inf_{g \in \cG}L_{h}^{1:k}(D_h^{1:k},f^k, g)\ge 5\beta,\  \mbox{for some}\ h \in [H],
\end{align}
where $\beta$ is a logarithmic confidence parameter. 
In fact, the left-hand side of  \eqref{eq:updating rule} represents the in-sample discrepancy between $f^k$ and the historical    data $D_h^{1:k}$ at step $h$ after the first $k$ episodes. 
When \eqref{eq:updating rule} does not hold, then we have 
\begin{align}\label{eq:not-updating rule}
    L_{h}^{1:k}(D_h^{1:k}, f^k, f^k)-\inf_{g \in \cG}L_{h}^{1:k}(D_h^{1:k},f^k, g)\leq 5\beta,\qquad \forall h \in [H]. 
\end{align}
Moreover, for any $k \in [K]$, 
 we have $f^{k} \subseteq \mathscr{B}^{k-1}$. 
 Moreover, let $t_{k-1}$ be the index of the episode after which $\mathscr{B}^{k-1} $ is constructed. 
 That is, $t_{k-1}$ is the smallest $t$ such that $ \mathscr{B}^{t} = \mathscr{B}^{k-1}$. 
Then by the construction of the confidence set $\mathscr{B}^{k-1}$,  the discrepancy between $f^{k}$ and the historical data $D_h^{1:t_{k-1} }$ satisfies  
\begin{align} \label{eq:update_confidence_set}
       L_h^{1:t_{k-1}}(D_h^{1:t_{k-1}}, f^{k}, f^{k})-\inf_{g \in \cG}L_h^{1:t_{k-1}}(D_h^{1:t_{k-1}}, f^{k},g)\le \beta. 
\end{align}
Comparing \eqref{eq:not-updating rule}  and \eqref{eq:update_confidence_set}, 
we observe that, when adding the  new data from $(t_{k-1} + 1)$-th to the $k$-th episode, the discrepancy between the collected data and $f^k$ remains  relatively small for all steps $h \in [H]$. 
In this case,   the improvement brought from adding new data limited, and thus we choose not to update the policy to save computation. 
Instead, when \eqref{eq:updating rule} holds, 
this means that the discrepancy between $f^k$ and the offline data $D_h^{1:k}$ is significant. 
In light of \eqref{eq:update_confidence_set}, the newly added data 
from $(t_{k-1} + 1)$-th to the $k$-th episode brings considerable 
 new information from newly collected data, and thus we update the confidence set and hence update the policy.  
Furthermore, in the following theorem, 
we prove that such a lazy policy switching scheme achieves both sample efficiency   while incurring a small  switching cost, assuming the underlying model belongs to the  
 $\ell_2$-type EC class.  
The detailed proof of the theorem is provided in  \S \OnlyInFull{\ref{sec:proof}}\OnlyInShort{D}.

\begin{theorem}\label{thm:l2}
    Given an EC class $(M,\cF,\cG,\ell,d,\kappa)$ with two hypothesis classes $\cF, \cG$ and a decomposable loss function $\ell$ satisfying Eq. \eqref{eq:l2cond}, Eq. \eqref{def:loss dominance} and Definition \ref{def:DLF}. Set $\beta = c(R^2\iota+R)$ for a large constant $c$ with $\iota = \log(HK^2\cN_\cL(1/K)/\delta)$, in which $\cN_\cL(1/K)$ is the $1/K$-covering number for DLF class $\cL = \{\ell_{h ,f'}(\cdot,\cdot, f, g): (h,f',f,g) \in [H]\times \cF\times \cF\times \cG\}$ with norm $\Vert\cdot \Vert_\infty$ (Defined in \S \OnlyInFull{\ref{appendix:define covering number}}\OnlyInShort{B}). With probability at least $1-\delta$,  Algorithm~\ref{alg:ET-Rare switch} achieves a sublinear regret
    \begin{align*}
        R(K) \le \widetilde{\cO}(\kappa H\sqrt{d\beta K}\cdot \mathrm{poly}(\log K)),\nonumber
    \end{align*}
    Also, Algorithm~\ref{alg:ET-Rare switch} has a logarithmic switching cost
    \begin{align}
        N_{\mbox{switch}}(K) \le \cO(d H\cdot\log K).
    \end{align}
\end{theorem}

\vspace{1em}

The theorem above gives us the upper bound for both a $\widetilde{\cO}(\sqrt{K})$ regret and a logarithmic switching cost. When applying to the specific examples such as function class with low $D_\Delta$-type BE dimension $d = d_{BE}(\cF,D_\Delta,1/\sqrt{T})$, Algorithm \ref{alg:ET-Rare switch} achieves a $\widetilde{\cO}(H\sqrt{d\beta K}\log(K)) = \widetilde{\cO}(H\sqrt{d\log(N_\cL(1/K)/\delta) K}\cdot \mathrm{poly}(\log K))$ regret and a $\cO(dH\cdot \log K)$ switching cost. Some specific examples of $\ell_2$-type EC class such as linear mixture MDP, KNR, and $\cD_\Delta$-type BE dimension,  and the corresponding theoretical results are provided in Section \ref{sec:Examples}.


\noindent\textbf{Comparison with Previous Algorithms} The common way to achieve low switching-cost problems is to measure the information gain and change the policy only when the gained information is large enough. However, 
 the previous techniques to represent the gained information cannot apply to more general RL problems. For the tabular MDP and linear MDP, the gain of new information can be explicitly formulated as the determinant of the Hessian matrix of the least-squares loss function. For the function classes with low eluder dimension, their algorithm requires the construction of the bonus function and a sensitivity-based subsampling approach, which cannot be extended beyond their setting. Moreover, they require a value closeness assumption: For each function $V:\cS \to [0,H]$,  they assume the function class $\cF$
 satisfies that $r(s,a) + \sum_{s'}\PP(s'\mid s,a)V(s') \in \cF$ for all $(s,a).$ This assumption is very stringent and is not satisfied by many general classes such as linear mixture MDP and KNR. All of these approaches cannot be applied to our EC class.

The computational complexity mainly depends on the Line~\ref{line:oracle} or \ref{line: oracle MG}  in Algorithm~\ref{alg:ET-Rare switch}. 
Previous works often assume there exists an oracle that approximately solves Line~\ref{line:oracle}, e.g., \citep{jin2021bellman, chen2022abc}. 
Such an oracle is queried in each episode to update the policy. 
Thus, these works incur an $\cO(K)$ oracle complexity.  
In contrast, with the lazy update scheme specified in Lines \ref{line:l2confidenceset}--\ref{line:updateend}, the oracle complexity of Algorithm \ref{alg:ET-Rare switch} is $\cO(\log K )$, which leads to an exponential improvement in terms of the computational cost. In the experiment, the execution time of our algorithm is 20 times faster than the algorithm without lazy policy switches, while maintaining a similar performance.

\section{Batch Learning Problem} \label{sec:batch}

In this section, we provide an algorithm for the batch learning problem. Recall that in the batch learning problem,  the agent selects the batch before the algorithm starts, then she uses the same policy within each batch.
Denote the number of batches as $B$, Algorithm~\ref{alg:EC_batch} try to divide each batch equally and choose the batch as $[k_i, k_{i+1})$, where  $k_i = i\cdot \lfloor K/B \rfloor + 1.$ This selection is intuitive and common in many previous works of batch learning~\citep{han2020sequential,quanquangu2021linearlowcost, gu2021batched}.
After setting the batches, the agent adopts   optimistic planning for policy updates, and only updates the policies in episodes $\{ k_i, i \in [B-1]\}$. 
The details of the algorithm is presented in Algorithm \ref{alg:EC_batch}.

In the following theorem, we provide a regret upper bound for Algorithm \ref{alg:EC_batch}.

\begin{theorem}\label{thm:batchl2}
    Given an  EC class $(M,\cF,\cG,\ell,d,\kappa)$ 
 with two hypothesis classes $\cF, \cG$ and a decomposable loss function $\ell$ satisfying \ref{eq:l2cond} and \ref{def:loss dominance}. Set $\beta = c(R^2\iota+R)$ for a large constant $c$ with $\iota = \log(HK^2\cN_\cL(1/K)/\delta)$, in which $\cN_\cL(1/K)$ is the $1/K$-covering number for DLF class $\cL = \{\ell_{h ,f'}(\cdot,\cdot, f, g): (h,f',f,g) \in [H]\times \cF\times \cF\times \cG\}$ with norm $\Vert\cdot \Vert_\infty$. With probability at least $1-\delta$ the Algorithm~\ref{alg:EC_batch} will achieve a sublinear regret
    \begin{align}
        R(T) \le \tilde{\cO}\left(\kappa H\sqrt{d\beta K}\log K+\kappa \cdot \frac{dHK}{B} \cdot (\log K)^2\right).\nonumber
    \end{align}
\end{theorem}
Hence if we choose $B = \Omega(\sqrt{K/d})$, we can get a sublinear regret $\widetilde{\cO}(H\sqrt{d\beta K})$. 
In particular, for the linear MDP with dimension $d_{\mathrm{lin}}$, we have $d = \widetilde{\cO}(d_{\mathrm{lin}})$ and $\beta = \widetilde{\cO}(d_{\mathrm{lin}}\cdot \mathrm{poly}(\log K))$, 
Theorem \ref{thm:batchl2} achieves a $\widetilde{\cO}(Hd_{\mathrm{lin}}\sqrt{K} + d_{\mathrm{lin}}HK/B)$ regret upper bound, which matches the regret lower bound  established  in \cite{gao2021lowcostlinear}. 
More specific examples and the corresponding results are provided in Section \ref{sec:Examples}. We also provide the batch learning results of $\ell_1$-EC class in \S \OnlyInFull{\ref{appendix:algorithm for l1}}\OnlyInShort{C.2}.

Now we provide the intuition about why the algorithm works, and the detailed proof is in \S \OnlyInFull{\ref{appendix:proof of batch}}\OnlyInShort{E}. First, for a batch $j$, we consider the maximum in-sample error brought by this batch: $$\max_{k \in [k_j,k_{j+1}-1]} L_h^{k_j:k}(D_h^{k_j:k}, f^{k_j}, f^{k_j}) - L_h^{k_j:k}(D_h^{k_j:k}, f^{k_j}, \cT(f^{k_j}))\triangleq c_j\beta.$$ Indeed, this term represents the maximum fitting error for the data within the data of this batch and the model $f^{k_j}$.  Then for all  batches $[k_j,k_{j+1}-1]$ with a small in-sample error, namely, $c_j = O(1)$,
 we can still deploy the optimism mechanism and control the regret, thus the final regret can vary in magnitude by at most a constant. Moreover, for these batches with $c_j \le 5$, we call them the "Good" batches, meaning that the regret caused by these batches can still be upper bounded by $\cO(\sqrt{K})$. For batch $j$ with $c_{j}>5$, we called them the "Bad" batches. Then we can show a fact that the number of "Bad" batches is at most $\cO((\log K)^2)$. 
In fact,  we can divide all the $c_j \in [5,K]$ into $\cO(\log K)$ intervals $[5\cdot 2^i, 5\cdot 2^{i+1})_{i\ge 0}$, and use the $\ell_2$-type eluder condition to bound that $|\{j\mid C/2\le c_j\le C\}|\le \cO(\log K)$ for any constant $C$.

Once the fact is proven, the regret can be derived by adding "Good" batches and "Bad" batches. All "Good" batches will lead to at most a $\cO(\sqrt{K})$ regret, and all "Bad" batches will lead to at most $\cO((K/B)\cdot(\log K)^2)$ regret.   Combining two types of batches, we can get Theorem~\ref{thm:batchl2}.

Moreover, we consider another batch learning setting called "the adaptive batch setting" that was studied in \cite{gao2019batched}. In this setting, the agent can select the batch size adaptively during the algorithm. At the end of each batch, the agent observes the reward feedback of this batch, and she can select the next batch size according to the historical information and change the policy. We show that in this setting, $\cO(\mbox{poly}(\log K))$ batches are sufficient for a $\cO(\sqrt{K})$ regret. 
The proof employs an extra  doubling trick performed on the low switching cost 
 Algorithm~\ref{alg:ET-Rare switch} and we discuss it in \S \OnlyInFull{\ref{appendix:adaptive batch}}\OnlyInShort{E.2}.

\section{Experiment}
We experimented in the linear mixture MDP with the same setting as \citep{chen2022abc}. We choose $T=2000$ and $\beta = 0.3\log T$ in the experiment, and the regret and the cumulative reward curves show that our algorithm maintains a sublinear regret which is slightly larger than OPERA algorithm in \citep{chen2022abc}. However, the average number of strategy transitions and calls to the optimization tool decreases from 2000 to 92.8 times over 10 simulations, decreasing the average execution time from 321.6 seconds to 15.9 seconds.
\begin{figure}
    \centering
    \includegraphics[scale=0.6]{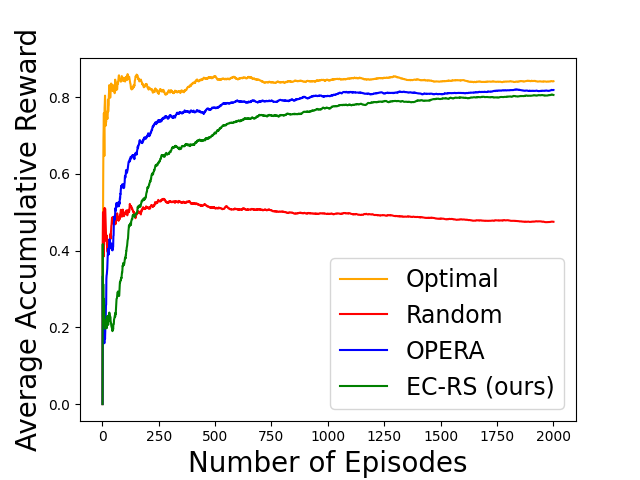}
    \caption{The average accumulative reward for optimal policy, random policy, OPERA algorithm \citep{chen2022abc} and EC-RS (Algorithm \ref{alg:ET-Rare switch})}
    \label{fig:exp}
\end{figure}


\section{Conclusion}
In this paper, we study the general sequential decision-making problem under general function approximation with two adaptivity constraints: the rare policy switch constraint and the batch learning constraint.
 Motivated by the $\ell_2$-eluder argument, we first introduce a general class named EC class that includes various previous RL models, and then provide algorithms for both two adaptivity constraints. For the rare policy switch problem, we propose a lazy policy switch strategy to achieve a low switching cost while maintaining a sublinear regret. For the batch learning problem, we analyze the regret when the batch is a uniform grid and get state-of-the-art results that match the lower bound under the linear MDP \citep{gao2021lowcostlinear}.
To the best of our knowledge, this paper is the first work to systematically investigate these two adaptivity constraints under a general framework that contains a wide range of RL problems.


\OnlyInFull{
\newpage
\appendix

\centerline{\begin{Large} \textbf{Appendix}\end{Large}}

\begin{large} \startcontents
\printcontents{}{1}{}\end{large}


\section{Related Work}\label{appendix: related work}
\paragraph{RL with General Function Approximation}
To solve the large-state RL problems, many works consider capturing the special structures of the MDP models. 
\cite{jiang@2017} consider  the RL problems with low Bellman rank; \cite{jin2020provably} consider  a particular linear structure of MDP models named {\em linear} MDP; 
\cite{wang2020reinforcement} consider RL problems with bounded {\em Eluder dimension} using sensitivity sampling, and \cite{ishfaq2021randomized} use a simpler optimistic reward sampling to combine the optimism principle and Thompson sampling. \cite{jin2021bellman} capture  an extension of Eluder dimension called {\em Bellman Eluder dimension}; \cite{du2021bilinear}  consider  a particular model named \textit{bilinear} model; Very recently, \cite{chen2022abc} consider a more extensive class ABC that contains many previous models. \cite{foster2021statistical,agarwal2022model, zhong2022gec} provide the  posterior-sampling style algorithm for sequential decision-making. Previous works also consider the Markov Games with multiple players under the function approximation setting. \cite{xie2020learning,huang2021markovgamegeneral,qiu2021reward,jin2022power,zhao2022provably,liu2023objective} study the two-player zero-sum Markov Games with linear or general function approximations.  \cite{zhong2021can,ding2022independent,cui2023breaking, wang2023breaking, foster2023complexity} further consider the general-sum Markov Games with function approximations. Among these works, our paper is particularly related to the works that use the eluder dimension to capture the complexity of a function class \citep{jin2021bellman,chen2022abc, liu2022partially,liu2022optimistic}. Compared to them, our EC class requires a property that is slightly stricter than the normal eluder argument, which can help us to deal with the additional adaptivity constraints. More details are provided in Section \ref{sec:ET}. 

\paragraph{RL with Adaptivity Constraints}
The rare policy switch problem and the batch learning problem are the two main adaptivity constraints considered in the previous works.
\cite{abbasi2011improved} first give  an algorithm to achieve a $\cO(\sqrt{K})$ regret and a $\cO(\log K)$ switching cost in the bandit problem. \cite{baiyu2019lowswitchcosttabular}  study the rare policy switch problem for tabular MDP and \cite{zhang2020lowswitchcosttabular}  improve their results. \cite{quanquangu2021linearlowcost,gao2021lowcostlinear} provide  low-switching cost algorithms for linear MDP. \cite{kong2021eluderlowcost} consider this problem for the function classes with low eluder dimension, and \cite{zhuoranyang2022lowcosteluder} extend their results to be gap-dependent. 
 \cite{qiao2023logarithmic} give an algorithm to achieve logarithmic switching cost in the linear complete MDP with low inherent Bellman error and generalized linear function approximation. They leave  the low switching cost problem for the function classes with low BE dimension as an open problem, where our paper solves a part of this open problem ($D_\Delta$-type BE dimension). 

For the batch learning problem,  \cite{perchet2016batched} consider  this problem in 2-armed bandits, and \cite{gao2019batched} further consider this problem in the multi-armed bandit with both fixed batch size and adaptive batch size. \cite{han2020sequential} study batch learning problem in the linear contextual bandits. \cite{gu2021batched} also study this problem in the  neural bandit setting. \cite{quanquangu2021linearlowcost} consider this problem in the linear MDP, and gives a lower bound of the  regret. As we will show later, our results match their lower bound in the linear MDP.

\section{Definition of Covering Number and Bracketing Number}\label{appendix:define covering number}
In the Theorem \ref{thm:l2}, Theorem \ref{thm:batchl2}, Theorem \ref{thm:l1} and Theorem \ref{thm:batchl1}, the regret result contains the logarithmic term of the $1/K$-covering number of the function classes $\cN_\cF(1/K)$ or the $1/K$-bracketing number $\cB_\cF(1/K)$. Both of them can be regarded as a surrogate  cardinality of the function class $\cF$. 

First, we provide the definition of $\rho$-covering number. In this work, we mainly consider the covering number with respect to the distance $\ell_\infty$ norm.

\begin{definition}[$\rho$-Covering Number]
    The $\rho$-covering number of a function class $\cF$ is the minimum integer $t$ that satisfies the following property: There exists $\cF'\subseteq \cF$ with $|\cF'|=t$, and for any $f_1 \in \cF$ we can find $f_2 \in \cF'$ such that $\Vert f_1-f_2\Vert_\infty \le \rho.$
\end{definition}
For $\ell_1$-EC class, Theorem \ref{thm:l1} uses the bracket number to approximate the cardinality as in the previous works \citep{zhan2022pac, zhong2022gec}. 
\begin{definition}[$\rho$-Bracket Number]
A $\rho$-bracket with size $N$ contains $2N$ functions $\{f_1^i, f_2^i\}_{i=1}^N $ that maps a policy $\pi$ and a trajectory $\tau$ to a real value such that $\Vert f_1^i(\pi,\cdot)-f_2^i(\pi,\cdot)\Vert_1 \le \rho$. Moreover, for any $f \in \cF$, there exists $i \in [N]$ such that $f_1^i(\pi,\tau) \le \PP_f^\pi(\tau) \le f_2^i(\pi,\tau)$. The $\rho$-bracket number of a function class $\cF$, denoted as $\cB_\cF(1/K)$, is the minimum size $N$ of a $\rho$-bracket.
\end{definition}
As shown in \cite{zhan2022pac}, the logarithm of the $1/K$-bracket number $\log(\cB_\cF(1/K))$ usually scales polynomially with respect to the parameters of the problem. 

\section{Pseudo-code of Algorithm for Batch Learning}
In this subsection, we provide the pseudo-code of our algorithm $\ell_2$-EC-batch for batch learning, which divides the entire episode into a uniform grid. The pseudo-code of the algorithm is provided in Algorithm \ref{alg:EC_batch}.
\begin{algorithm}[t]
    \begin{algorithmic}[1]
        
	\caption{$\ell_2$-EC-Batch}
	\label{alg:EC_batch}
	\STATE {\textbf{Input}} $D_1,D_2,\cdots,D_H=\emptyset,\mathscr{B}_1 = \cF$.
	
	\FOR{$k=1,2,\cdots,K$}
            
	    \STATE \textcolor{blue}{(MDP):} Compute $\pi^k = \pi_{f^k}$, where $f^k = \arg\max_{f \in \mathscr{B}^{k-1}}V_{f}^{\pi_f}(s_1)$.\label{line:batch_oracle}
	    \STATE \textcolor{blue}{(Zero-Sum MG):} Compute $\upsilon^k = \upsilon_{f^k}$, where $f^k = \arg\max_{f \in \mathscr{B}_{k-1}}V_{f}^{\upsilon_f, \mu_f}(s_1)$. The adversary chooses strategy $\mu^k$, then we let $\pi^k = (\upsilon^k, \mu^k).$\label{line: oracle MG}
	    \STATE Execute policy $\pi^k$ to collect  the trajectory, update $D_h = D_h\cup\{\zeta_h^k, \eta_h^k\}, \forall h\in [H]$.

            
	    \IF {$k = i\cdot \lfloor K/B\rfloor + 1$ for some $i \ge 0$} \label{alg:batch_l2beginif}
     
        \STATE Update \begin{align*}
	        \mathscr{B}^{k}=\left\{f \in \cF: L_h^{1:k}(D^{1:k}_h, f, f)-\inf_{g \in \cG}L_h^{1:k}(D^{1:k}_h,f, g)\le \beta, \forall h \in [H] \right\}.
	    \end{align*}\label{line:batch_l2confidenceset}
            \ELSE \STATE $\mathscr{B}^{k}=\mathscr{B}^{k-1}$.    \ENDIF\label{alg:batchl2endif}
	\ENDFOR
 \end{algorithmic}
\end{algorithm}

\section{Concrete Examples and Theoretical Results of EC Class}\label{sec:Examples}
Now we provide a large amount of RL problems that are contained in the  EC class with a decomposable loss function and provide corresponding theoretical results for them. Additional examples are provided in \S \OnlyInFull{\ref{appendix:additional examples}}\OnlyInShort{F}.
\subsection{Bellman Eluder Dimension}

\begin{example}[$D_\Delta$-type Bellman Eluder Dimension]
$D_\Delta$-type Bellman eluder dimension \citep{jin2021bellman} can be regarded as a general extension of the eluder dimension. The class with low $D_\Delta$-type Bellman eluder dimension subsumes a wide range of classical RL problems such as linear MDP and MDP with low eluder dimension. 
If we choose $\zeta_h = \{s_{h+1}\}, \eta_h = \{s_h,a_h\}$, operator $\cT=\{\cT_h\}_{h \in [H]}$ as the Bellman operator in \eqref{eq:bellman operator for MDP}, and the loss function \begin{align}\label{eq:DLF selection be dimension}
    \ell_{h,f'}(\zeta_h, \eta_h, f, g) = Q_{h,g}(s_h,a_h) - r(s_h,a_h) - V_{h+1,f}(s_{h+1}),
\end{align} the loss function satisfies the decomposable property:
\begin{align*}
    &\ell_{h,f'}(\zeta_h, \eta_h,f,g) - \EE_{\zeta_h}\Big[\ell_{h,f'}(\zeta_h, \eta_h,f,g)\Big]\\ & \qquad = (Q_{h,g}(s_h,a_h) - r(s_h,a_h) - V_{h+1,f}(s_{h+1}))-(Q_{h,g}(s_h,a_h) - \cT_hV_{h+1}(s_h,a_h))\\
    &\qquad =(\cT_hV_{h+1}(s_h,a_h)-r(s_h,a_h) - V_{h+1,f}(s_{h+1}))\\
    &\qquad=\ell_{h,f'}(\zeta_h, \eta_h, f, \cT(f)),
\end{align*}
and $\cT(f^*) = f^*$ by Bellman equations. 
The following lemma shows that $D_\Delta$-type Bellman eluder dimension model belongs to $\ell_2$-type EC class with parameter $d = d_{BE}(\cF, D_\Delta, 1/\sqrt{T})$.
\begin{lemma}[Low $D_\Delta$-type Bellman Eluder Dimension $\subset$ $\ell_2-$type EC class]\label{lemma:BEdim l2}
    Suppose the function class $\cF$ with a low $D_\Delta$-type Bellman eluder dimension with auxiliary function class $\cG$ \citep{jin2021bellman}, then for any MDP model $M$, choose $\zeta_h = \{s_{h+1}\}, \eta_h = \{s_h,a_h\}$ and $(M,\cF,\cG,)$, $\kappa=1$,  and DLF $\ell$ as in \ref{eq:DLF selection be dimension}, then $(M,\cF,\cG,\ell,d_{BE}(\cF, D_\Delta, 1/\sqrt{T}),\kappa)$ is a  $\ell_2$-type EC class by $\ell_2$ condition:
    If $\sum_{i=1}^{k-1}[\cE(f^k,s_h^i,a_h^i)^2]\le \beta$  holds for any $k \in [K]$ and $\beta \ge 9$,  then for any $k \in [K]$ we have \begin{align}\label{eq:BEdiml2}\sum_{i=1}^k [\cE(f^i,s_h^i,a_h^i)^2]\le \cO(d\beta\log K),
    \end{align}
    where $d = d_{BE}(\cF, D_\Delta, 1/\sqrt{T})$ and we choose the upper bound of the Bellman residual $\cE$ as $R = 3$.
    The dominance can be derived by Lemma 1 in \cite{jiang@2017}.
\end{lemma}
Combining Theorem~\ref{thm:l2}, Theorem~\ref{thm:batchl2} and Lemma~\ref{lemma:BEdim l2}, for the rare policy switch problem, Algorithm \ref{alg:ET-Rare switch} achieves a 
$\widetilde{\cO}(H\sqrt{d_{\mathrm{BE}}(\cF,D_\Delta, 1/\sqrt{T})\beta\cdot \log K})$ regret and a  $\widetilde{\cO}(d_{\mathrm{BE}}(\cF,D_\Delta, 1/\sqrt{T})H\cdot \log K)$ switching cost, where $\beta$ is defined in Theorem \ref{thm:l2}. For the batch learning problem with $B$ batches, 
Algorithm \ref{alg:EC_batch} provides a $\widetilde{\cO}(H\sqrt{d_{\mathrm{BE}}(\cF,D_\Delta, 1/\sqrt{T})\beta K }\log K + d_{\mathrm{BE}}(\cF,D_\Delta, 1/\sqrt{T})HK(\log K)^2/B)$ regret upper bound. Since low $D_\Delta$ type BE dimension implies low eluder dimension, compared to \cite{kong2021eluderlowcost}, we can derive a $\sqrt{K}$ regret and $\cO((\log K)^2)$ switching cost \textit{without a restrictive value-closeness assumption}.
\end{example}
\subsection{Linear Mixture MDP}

\begin{example}[Linear Mixture MDP] The transition kernel of the Linear Mixture MDP~\citep{ayoub2020model} is a linear combination of several basis kernels. In this model, the transition kernel can be represented by $\PP_h(s'\mid s,a) = \langle \phi(s,a,s'), \theta_h\rangle$, where the feature $\phi(s,a,s') \in \RR^d$ is known and the weight $\theta_h \in \RR^d$ is unknown. We assume $\Vert \theta_h \Vert_2 \le 1$ and $\Vert \phi(s,a,s')\Vert_2 \le 1$. The reward function can be written as $r_h = \langle\psi(s,a), \theta_h\rangle$ with mapping $\psi(s,a):\cS\times \cA\mapsto \RR^d$. The following lemma shows that $\ell_2$-type EC class contains the linear mixture MDP as a special case.
\begin{lemma}[Linear Mixture MDP $\subset \ell_2$-type EC Class]\label{lemma:linear mixture l2}
    The linear mixture model belongs to $\ell_2$-type EC class. Indeed, if we choose the model-based function approximation $\cF = \cG = \{\theta_h\}_{h=1}^H$, $\zeta_h = \{s_{h+1}\}, \eta_h = \{s_h,a_h\}$,  $\kappa=1$, and $$\ell_{h,f'}(\zeta_h, \eta_h, f, g) = \theta_{h,g}^T\left[\psi(s_h,a_h)+ \sum_{s'}\phi(s_h, a_h, s')V_{h+1, f'}(s')\right]-r_h-V_{h+1,f'}(s_{h+1}),$$
    then we have
    \begin{align}
    \EE_{\zeta_h}\Big[\ell_{h,f'}(\zeta_h, \eta_h, f, g)\Big] = (\theta_{h,g}-\theta_h^*)^T \left[\psi(s_h,a_h)+ \sum_{s'}\phi(s_h, a_h, s')V_{h+1,f'}(s')\right],\label{eq:linear mixture expectation loss}
\end{align}
    and the loss function satisfies the dominance, decomposable property with $\cT(f) = f^*$ and $\ell_2$-type condition. Hence $(M,\cF,\cG,\ell,d,\kappa)$ is a $\ell_2$-type EC class.
\end{lemma}
Combining Theorem \ref{thm:l2}, Theorem \ref{thm:batchl2} and Lemma \ref{lemma:linear mixture l2}, Algorithm \ref{alg:ET-Rare switch} provides a $\widetilde{\cO}(H\sqrt{d\beta \log K})$ regret and a $\widetilde{\cO}(dH\log K)$ switching cost. 
Also, Algorithm \ref{alg:EC_batch} satisfies  a $\widetilde{\cO}(H\sqrt{d\beta K}\log K + dHK(\log K)^2/B)$ regret upper bound, where  $B$ is the  number of batches. 
\end{example}

\subsection{Kernelized Nonlinear Regulator}
Kernelized Nonlinear Regulator \citep{kakade2020information} (KNR) models a nonlinear control system as an unknown function in a RKHS.   When we consider the finite dimension RKHS, the model represents the dynamic as $
s_{h+1}=U_{h}^*\phi(s_h,a_h)+\varepsilon_{h+1}$, where $\phi(s_h,a_h):\cS\times \cA\to \RR^{d_\phi}$ is a mapping from a state-action pair to a feature with dimension $d_\phi$ and $\Vert \phi(\cdot,\cdot)\Vert_2 \le 1$. $U_h^* \in \RR^{d_s}\times \RR^{d_\phi}$ is a linear mapping such that $\Vert U_h^*\Vert_2 \le R$ and $\varepsilon_{h+1}\sim\cN(0,\sigma^2I)$ is a normal distribution noise.
Following \cite{chen2022abc}, we choose $\cF_h = \cG_h = \{U \in \RR^{d_s}\times \RR^{d_\phi}:\Vert \cU\Vert_2 \le R\}$ and $\ell_{h,f'}(\zeta_h, \eta_h, f, g) =U_{h,g}\phi(s_h,a_h)-s_{h+1}$\footnote{Note that the loss function can be arbitrary large because of normal distribution noise $\varepsilon$, so it does not satisfy the bounded requirement. However, we can regard it as a bounded loss function because it can be upper bounded by $\widetilde{\cO}(\sigma)$ with high probability \citep{chen2022abc}.} is a DLF and satisfies the $\ell_2$-type condition. Denote $\EE_{\zeta_h}[\ell_{h,f'}(\zeta_h, \eta_h, f,g)] = (U_{h,g}-U^*)\phi(s_h,a_h)$.
\begin{lemma}[KNR $\subset$ $\ell_2$-type EC class]\label{lemma:KNRl2}
    If we choose $\cF, \cG, \ell$ as defined above, $(M,\cF,\cG,\ell,d,\kappa)$ is a $\ell_2$-type EC class, where  $d = \tilde{\cO}(d_\phi)$ is a parameter and $\kappa = 2H/\sigma.$
     Indeed, fix a parameter $\beta \ge R^2$, the KNR model belongs to $\ell_2$-type EC class by:
    \begin{align}
        \sum_{i=1}^{k-1}&\Vert(U_{h,f^k}-U^*)\phi(s_h^i,a_h^i)\Vert_2^2\le \beta, \ \ \forall \ k \in [K],\\&\Rightarrow
        \sum_{i=1}^k \Vert(U_{h,f^i}-U^*)\phi(s_h^i,a_h^i)\Vert_2^2 \le \cO(d\beta\log K),\ \ \forall \ k \in [K]. \label{KNR}
    \end{align}
    The proof of Eq.\eqref{KNR}, decomposable property and the dominance property with $\kappa = 2H / \sigma$ are provided in Proposition 11 of \cite{chen2022abc}. Also, if we choose $X_h(f^k) = (U_{h,f^k}-U^*)$ and $W_{h,f^k}(s_h,a_h) = \phi(s_h,a_h)$, we can apply the similar argument in Section \ref{sec:proof of linear mixture} to prove the $\ell_2$-type condition.
\end{lemma}

In \S \OnlyInFull{\ref{appendix:additional examples}}\OnlyInShort{F}, we also provide some additional examples including linear $Q^*/V^*$, Linear Quadratic Regulator (LQR), generalized linear Bellman complete class.

\section{Switching-Cost for Zero-Sum Markov Games}\label{def of zero-sum MG}
\subsection{Definition}
\paragraph{Markov Games}
A zero-sum Markov Game (MG) consists of two players, while max-player P1 wants to maximize the reward and min-player P2 wants to minimize it. The model is represented by a tuple $(\cS, \cA, \cB, H, \PP(\cdot \mid s,a,b), r(s,a,b))$, where $\cA$ and $\cB$ denote the action space of player P1 and P2 respectively. Similar to the episode MDP, we also assume $\sum_{h=1}^H r_h(s_h,a_h,b_h)\in [0,1]$ for all possible sequences $\{s_h,a_h,b_h\}_{h=1}^H$ in this paper. The policy pair $(\upsilon, \mu) = \{\upsilon_h, \mu_h\}_{h \in [H]}$ consists of $2H$ functions $\upsilon_h : \cS\to \Delta_{\cA}, \mu_h:\cS\to\Delta_{\cB}$.
For any policy $(\upsilon, \mu)$, the action value function and state value function can be represented by 
\begin{gather*}
    Q^{\upsilon, \mu}_h(s,a,b) := \EE_{\upsilon,\mu} \Bigg[\sum_{h' = h}^H r_{h'}(s_{h'}, a_{h'}, b_{h'}) \Bigg|s_{h}=s, a_{h}=a, b_h = b \Bigg],\\
    V^{\upsilon, \mu}_h(s) := \EE_{\upsilon,\mu} \Bigg[\sum_{h' = h}^H r_{h'}(s_{h'}, a_{h'}, b_{h'}) \Bigg|s_{h}=s \Bigg],
\end{gather*}
Given the policy of P1 $\upsilon$, the {\em best response policy} of P2 is $\mu_{\upsilon}^* = \arg\min_{\mu} V_1^{\upsilon, \mu}(s_1)$. Similarly, the best response of P1 given $\mu$ is $\upsilon_{\mu}^* = \arg\max_{\upsilon}V_1^{\upsilon, \mu}(s_1)$. The Nash Equilibrium (NE) of an MG is a policy pair $(\upsilon^*, \mu^*)$ such that 
\begin{align*}
    V_1^{\upsilon^*, \mu_{\upsilon^*}^*}(s_1) = V_1^{\upsilon^*, \mu^*}(s_1) = V_1^{\upsilon_{\mu^*}^*, \mu^*}(s_1).
\end{align*}
We denote $V_h^{\upsilon^*, \mu^*}$ and  $Q_h^{\upsilon^*, \mu^*}$ by  $V_h^* $ and $Q_h^*$ respectively in the sequel. In addition, to simplify the notation, we let $\pi = (\upsilon, \mu)$ denote the joint policy of the two players. 
Then we write 
$V_h^\pi(s_1) = V_h^{\upsilon ,  \mu}(s_1)$.
Similar to the MDP, we can define the  Bellman operator $\cT$ for a MG  by letting  \begin{align}( \cT_h Q_{h+1} )(s_h,a_h,b_h)= r_h(s_h,a_h,b_h) + \EE_{s_{h+1}} \bigl [  \max_{\upsilon}\min_{\mu} Q_{h+1} (s_{h+1}, \upsilon, \mu) \bigr ]  ,\label{eq:bellman operator for MG}\end{align}where we denote $Q _{h+1} (s,\upsilon,\mu)= \EE_{a\sim \upsilon, b\sim\mu} [ Q _{h+1} (s,a,b) ] $.
By definition, $Q^* = \{Q_h^* \}_{h\in [H]}$ is the unique fixed point of $\cT$, i.e., $Q_h^* = \cT_h Q_{h+1}^*$ for all $h \in [H]$.

\paragraph{Function Approximation of Zero-Sum Markov Games}
In specific, 
let $M$ denote the MG instance, which is clear from the context. We assume that we have access to a hypothesis class $\cF = \cF_1\times \cdots \times \cF_H$, where a  hypothesis function $f = \{ f_1, \ldots,  f_H \}  \in \cF $ 
either represents an action-value function $Q_f = \{ Q_{h,f}\}_{h\in [H]}$ in the model-free setting, or the environment model of zero-sum MG $M_f = \{ \PP_{h,f} ,  r_{h,f} \}_{h\in [H]}$ in the model-based setting. 
Similar to the function approximation of the single-agent MDP, under model-free zero-sum  MG setting,  we denote  $\pi_f = (\upsilon_f, \mu_f)$, where $(\upsilon_f, \mu_f)$ is a NE policy pair with respect to $Q_{f}(s,\cdot,\cdot)$. 
Moreover, given $Q_f$ and $\pi_f$, we define state-value function $V_f$ by letting 
$V_{h,f}(s) =  \EE_{a \sim \pi_{h,f}(s)}[Q_{h,f}(s,a)]$ in the MDP and $V_{h,f}(s) = \EE_{a,b \sim \pi_{h,f}(s)}[Q_{h,f}(s,a,b)]$ in the MG. 

\subsection{Learning Goal of Zero-Sum Markov Games}
For zero-sum MGs, we aim to design online reinforcement learning algorithms for the player P1 (max-player). 
In other words, we only control P1 and let P2 play arbitrarily. 
The goal is to design an RL algorithm such that P1's expected total return is close to the value of the game, namely $V^*_1 (s_1)$. 
For any $k \in [K]$, in the $k$-th episode,  players  P1 and P2 executes policy pair $\pi^k = (\nu^k, \mu^k) $ and P1's expected total return is given by $V_1^{\pi^k} (s_1)$. The regret of P1 is also given by \eqref{eq:define_regret}.

For MGs with the decoupled setting, since we can only control the player P1, the switching cost is only defined on the action of player P1 $\upsilon$, i.e. 
   $$ N_{\mbox{switch}}(K) = \sum_{k=1}^K \II\{\upsilon^k \neq \upsilon^{k+1}\}.$$ 

\subsection{Algorithms and Theoretical Results of Zero-Sum Markov Games}
The algorithm is provided in Algorithm \ref{alg:ET-Rare switch MG}. 
\begin{algorithm}[H]
    \begin{algorithmic}[1]
        
	\caption{$\ell_2$-EC-RS}
	\label{alg:ET-Rare switch MG}
	\STATE {\textbf{Initialize:}} $D_1,D_2,\cdots,D_H=\emptyset,\mathscr{B}_1 = \cF$.
	
	\FOR{$k=1,2,\cdots,K$}
            
	    \STATE \textcolor{blue}{(Zero-Sum MG):} Compute $\upsilon^k = \upsilon_{f^k}$, where $f^k = \arg\max_{f \in \mathscr{B}_{k-1}}V_{f}^{\upsilon_f, \mu_f}(s_1)$. The adversary chooses strategy $\mu^k$, then we let $\pi^k = (\upsilon^k, \mu^k).$
	    \STATE Execute policy $\pi^k$ to collect  the trajectory, update $D_h = D_h\cup\{\zeta_h^k, \eta_h^k\}, \forall h\in [H]$.

            
	    \IF {$L_h^{1:k}(D^{1:k}_h, f^k, f^k)-\inf_{g \in \cG}L_h^{1:k}(D^{1:k}_h,f^k, g)\ge 5\beta$ for some $h \in [H]$} 
     
        \STATE Update \begin{align*}
	        \mathscr{B}^{k}=\left\{f \in \cF: L_h^{1:k}(D^{1:k}_h, f, f)-\inf_{g \in \cG}L_h^{1:k}(D^{1:k}_h,f, g)\le \beta, \forall h \in [H] \right\}.
	    \end{align*}
            \ELSE \STATE $\mathscr{B}^{k}=\mathscr{B}^{k-1}$.
            \ENDIF
	\ENDFOR
 \end{algorithmic}
\end{algorithm}
An example of zero-sum MGs is the decoupled zero-sum MGs with a low minimax BE dimension. We formulate it in \ref{example:decoupled MG}, and prove that it belongs to $\ell_2$-type EC class.
\begin{example}[Decoupled Zero-Sum Markov Games with Low Minimax BE Dimension]\label{example:decoupled MG}
    The $\ell_2$-type EC class can be extended to the multi-agent setting. We consider the zero-sum MGs with the decoupled setting \citep{huang2021markovgamegeneral}, which means that the agent can only control the max-player P1, while an adversary can control the min-player P2. In this case, we let $\eta_h = \{s_h, a_h, b_h\}$ and $\zeta_h = \{s_{h+1}\}$, choose the loss function as 
    \begin{align}\label{eq:loss Markov}
        \ell_{h,f'}(\zeta_h, \eta_h, f, g) &= Q_{h,g}(s_h,a_h,b_h) - r_h(s_h,a_h,b_h)-V_{h+1,f}(s_{h+1})\\
        &=Q_{h,g}(s_h,a_h,b_h) - r_h(s_h,a_h,b_h) - \max_{\upsilon}\min_{\mu}Q_{h+1,f}(s_{h+1},\upsilon,\mu)\nonumber\\
        &\triangleq \cE(f, s_h,a_h,b_h)\nonumber.
    \end{align}
    The Bellman operator for MGs is defined as $$\cT_h'f(s_h,a_h,b_h)= r_h(s_h,a_h,b_h) + \EE_{s_{h+1}}\max_{\upsilon}\min_{\mu} f(s_{h+1}, \upsilon, \mu),$$ where we denote $f(s,\upsilon,\mu)= \EE_{a\sim \upsilon, b\sim\mu}[f(s,a,b)]$. We prove the decoupled MG belongs to the $\ell_2$-type EC class.
    \begin{lemma}[Decoupled Zero-Sum MGs $\subseteq$ $\ell_2$-type EC Class]\label{lemma:decouped markov l2 type}
        If we choose $\cF, \cG $ such that $\cT_h\cF \subseteq \cG$, then for any two-player zero-sum MG $M$, $(M,\cF,\cG,\ell, d,\kappa)$ is a $\ell_2$-type EC class, where $\eta_h = \{s_h, a_h, b_h\}$, $\zeta_h = \{s_{h+1}\}$, $\ell$ is chosen as in \ref{eq:loss Markov}, $\kappa=1$, and the parameter $d$ is  the minimax BE dimension $d_{\mathrm{ME}}(\cF, 1/\sqrt{T})$.
        The dominance and decomposable property of the loss function (Eq. \eqref{eq:loss Markov}) with the Bellman operator $\cT$ for MGs are provided in \cite{huang2021markovgamegeneral}. The $\ell_2$-eluder condition holds by replacing $a_h^i$ in Lemma~\ref{lemma:BEdim l2} to $(a_h^i,b_h^i)$: If $\sum_{i=1}^{k-1}[\cE(f^k,s_h^i,a_h^i, b_h^i)^2]\le \beta$ holds for any $k \in [K]$ and $\beta \ge R^2$, then for any $k \in [K]$ we have 
\begin{align}\label{eq:markovdiml2}\sum_{i=1}^k \Big[\cE(f^i,s_h^i,a_h^i, b_h^i)^2\Big]\le \cO(d\beta\log K).
    \end{align}
    \end{lemma}
\end{example}
Similarly, the previous theorems (Theorem \ref{thm:l2} and Theorem \ref{thm:batchl2}) and Lemma \ref{lemma:decouped markov l2 type} give a  $\widetilde{\cO}(H\sqrt{d\beta \log K})$ regret and $\widetilde{\cO}(dH\log K)$ switching cost for the decoupled zero-sum MG, where $d$ is the minimax BE dimension $d_{\mathrm{ME}}(\cF, 1/\sqrt{T})$. Also Algorithm \ref{alg:EC_batch} gives a $\widetilde{\cO}(H\sqrt{d\beta K}\log K + dHK(\log K)^2/B)$ regret.
We mainly consider the decoupled setting because it can be naturally contained in our $\ell_2$-type EC class.

\section{\texorpdfstring{$\ell_1$}{}-type EC Class}\label{appendix:l1 EC class}
\subsection{Definition of \texorpdfstring{$\ell_1$}{}-type EC Class}
In recent years, the $\ell_1$-eluder argument has been proposed in \citep{liu2022partially} for the sample-efficient algorithm of POMDP, and \citep{liu2022optimistic} generalize it to the more general classes.
Similar to $\ell_2$-type EC class, we provide the definition of the $\ell_1$-type EC class based on \citep{liu2022optimistic}. The $\ell_1$-type class has two assumptions, which are similar to the $\ell_2$-type EC class. 
To provide a consistent treatment of $\ell_2$-type EC class, we let $\{ \zeta_h , \eta_h\}_{h\in [H]} $ be subsets of the trajectory.  
In particular, we let $\eta_h = \{ \cT_H\}$ and $\zeta _h = \emptyset$, and only consider the single-agent MDP in the $\ell_1$-type EC class.
\begin{definition}
    
Given an MDP or POMDP instance (Example \ref{pomdp}) $M$, let $\cF$ and $\cG$ be two hypothesis function classes satisfying the realizability Assumption \ref{assum:realizability} with $\cF \subseteq \cG$. For any $h \in [H]$ and $f' \in \cF$, let $\ell_{h,f'}(\zeta_h,\eta_h,f,g)$ be a vector-valued loss function at step $h$, where $\zeta_h,\eta_h$ are subsets of trajectory that defined above. For parameters $d$ and $\kappa$, we say that $(M,\cF,\cG,\ell,d,\kappa)$ is a $\ell_1$-type EC class if the following two conditions hold for any $\beta$ and $h \in [H]$:

(i). ($\ell_1$-type Condition) For any $K$ hypotheses $f^1,\cdots,f^K \in \cF$, if \begin{align}\sum_{i=1}^{k-1} \mathbb{E}_{\eta_h\sim \pi^i, \zeta_h}\Big[\ell_{h,f^i}(\zeta_h, \eta_h, f^k,f^k)\Big] \le \sqrt{\beta k}\label{l1:precondition}\end{align} holds for any $k \in [K]$, then for any $k \in [K]$, we have \begin{align}\label{eq:l1-type condition}
          \sum_{i=1}^{k} \mathbb{E}_{\eta_h\sim \pi^i, \zeta_h}\Big[\ell_{h,f^i}(\zeta_h, \eta_h, f^i, f^i)\Big] \le \widetilde{\cO}(\mathrm{poly}\log (k)(\sqrt{d\beta k} + d\cdot \mathrm{poly}(H))).
    \end{align}
    When we choose $\beta \ge 1$, the right side of Eq.\eqref{eq:l1-type condition} can be simplified as $\widetilde{\cO}\left(\sqrt{d\beta k}\cdot \mathrm{poly}\log (k)\right)$
    
    (ii). ($\kappa$-Dominance) For any fixed $k \in [K]$, with probability at least $1-\delta$,
\begin{align} 
        \sum_{i=1}^k (V_{1,f^i}(s_1)-V^{\pi_i}(s_1))
        \le \kappa\cdot \left(\sum_{h=1}^H \sum_{i=1}^{k} \mathbb{E}_{\eta_h \sim \pi^i, \zeta_h}\Big[\ell_{h,f^i}(\zeta_h, \eta_h, f^i,f^i)\Big]\right).
    \end{align}
\end{definition}
Moreover, in this work we only consider a particular loss function 
\begin{align}\ell_{h,f'}(\zeta_h, \eta_h, f,g) = \ell_{h,f'}(\tau_H, f,g) = | \mathbb{P}_{f}(\tau_H)/\mathbb{P}_{f^*}(\tau_H)-1|\label{l1:loss}\end{align}
in the $\ell_1$-type EC class,
where $f^*$ is the true model in realizability Assumption~\ref{assum:realizability}, and \begin{align*}
    \PP_{f}(\tau_H) = \prod_{h=1}^H \PP_f(s_h\mid \tau_{h-1})
\end{align*} 
is the product of transition probability in $\tau_H$  under the model $f$.
Then \begin{align*}&\mathbb{E}_{\eta_h\sim \pi, \zeta_h}[\ell_{h,f'}(\zeta_h, \eta_h, f,g)] \\&\quad=  \mathbb{E}_{\eta_h\sim \pi, \zeta_h}\left(\frac{\prod_{h=1}^H \PP_f(s_h\mid \tau_{h-1})}{\prod_{h=1}^H \PP_{f^*}(s_h\mid \tau_{h-1})}-1\right)\\&\quad=\EE_{\tau_H\sim \pi}\left[\frac{\prod_{h=1}^H (\PP_f(s_h\mid \tau_{h-1}) \pi(a_h\mid s_h,\tau_{h-1}))}{\prod_{h=1}^H (\PP_{f^*}(s_h\mid \tau_{h-1})\pi(a_h\mid s_h,\tau_{h-1}))}-1\right]\\
&\quad=\sum_{\tau_H}\left[\prod_{h=1}^H (\PP_f(s_h\mid \tau_{h-1}) \pi(a_h\mid s_h,\tau_{h-1}))-\prod_{h=1}^H (\PP_{f^*}(s_h\mid \tau_{h-1})\pi(a_h\mid s_h,\tau_{h-1}))\right]s
\\&\quad=d_{\mathrm{TV}}(\mathbb{P}_{f}^{\pi}, \mathbb{P}_{f^*}^{\pi}),
\end{align*} which is the total variation difference between the trajectory distribution under model $f$ and the true model $f^*$ with policy $\pi$. By this particular selection of loss function, the $\kappa$-Dominance property is satisfied by $\kappa=1/H$ and the following inequality:
\begin{align*}
    \sum_{i=1}^k (V_{1,f^i}(s_1)-V^{\pi_i}(s_1)) \le  \sum_{i=1}^k  d_{\mathrm{TV}}(\PP_{f^i}^{\pi_i}, \PP_{f^*}^{\pi_i}).
\end{align*}

Compared to the $\ell_2$-type condition, the primary difference is that the precondition of $\ell_1$-type condition (Eq.\eqref{l1:precondition}) requires the sum of $\ell_1$ norm of the loss function can be controlled by $\cO(\sqrt{k})$, while the precondition of $\ell_2$-type condition (Eq.\eqref{eq:l2 condition precondition}) requires the square sum of the loss function is controlled by $\cO(\beta)$. Second, the selection of $\zeta_h$ and $\eta_h$ are different to the $\ell_2$-type EC class, and the left side of Eq.\eqref{eq:l1-type condition} contains an extra expectation on $\eta_h = \tau_H\sim \pi^i$. Moreover, since we consider a particular scalar loss function $\ell_{h,f'}(\zeta_h, \eta_h, f, g) = |\PP_f(\tau_H)/\PP_{f^*}(\tau_H)-1|$, we do not use the norm on the loss function like $\ell_2$-type condition. 
With this selection of loss function, the $\ell_1$-type Condition Eq.\eqref{eq:l1-type condition} is similar to the generalized eluder-type condition (Condition 3.1) in \cite{liu2022optimistic}.
We provide two examples in the $\ell_1$-type EC class, which are also introduced in the previous works \citep{liu2022partially, liu2022optimistic}. 
\begin{example}[Undercomplete POMDP \citep{liu2022partially}]\label{pomdp}
 
A partially observed Markov decision process (POMDP) is represented by a tuple $$(\cS, \rO, \cA, H, s_1, \TT = \{\TT_{h,a}\}_{(h,a) \in [H]\times \cA}, \OO = \{\OO_h\}_{h \in [H]}, r = \{r_h\}_{h \in [H]}),$$
where $\TT_{h,a}\in \RR^{|\cS| \times |\cS|}$ represents the transition matrix for latent state of the action $a$ at step $h$, $\OO_h:\cS\times \rO \mapsto \RR$ denotes the probability of generating the observation $o\in \rO$ conditioning on the latent state $s \in \cS$.  $r_h:\rO \mapsto \RR^+$ is the reward function at step $h$ with observation. We assume $\sum_{h=1}^H r_h(s_h,a_h)\in [0,1]$ for all possible sequences $\{s_h,a_h\}_{h \in [H]}.$ 
During the interactive process, at each step, the agent can only receive the observation and reward without information about the latent state. 
In POMDP, we consider the general policy $\pi = \{\pi_h\}_{h \in [H]}$, where $\pi_h:\tau_{h-1}\times \cS\to \Delta_{\cA}$, which can be history-dependent. At step $h$, the agent can only see her observations $o_h$ with probability $\OO_h(s_h,o_h)$, take her action $a_h$ with policy $\pi(\tau_{h-1}\times s_h)$,  and receive the reward $r_h(s_h,a_h)$. Then the agent arrives to the next state $s'$ with probability $\TT_{h,a_h}(\cdot \mid s_h).$ For POMDP, the transition kernel $\PP_f$ consists of $\{\TT_f,\OO_f\}$, and the model is represented by  $M_f = \{\TT_f, \OO_f,r_f\}$.

Undercomplete POMDP \cite{liu2022partially} is a special case of POMDP such that $S=|\cS|\le |\rO|$ and there exists a constant $\alpha>0$ with $\min_h \sigma_S(\OO_h)\ge \alpha.$ This assumption implies that the observation contains enough information to distinguish two states. In this paper, we only consider undercomplete POMDP because only in this setting we can have a sublinear regret result.\footnote{In the previous works studying sample-efficient POMDP, they only provide sample complexity or "pseudo-regret" (defined in \cite{liu2022optimistic}, \cite{zhong2022gec}) for overcomplete POMDP.}
    The undercomplete POMDP with the model classes $\cF$ and  $\min_{h}\sigma_S(\OO_h) \ge \alpha$ belongs to the $\ell_1$-type EC class by 
    \begin{align}\label{l1typecondition}
        \sum_{i=1}^{k-1} d_{\mathrm{TV}}&(\PP_{f^k}^{\pi^i}, \PP_{f^*}^{\pi^i}) \le \sqrt{\beta k}, \ \forall \ k \in [K],   \nonumber\\&\Rightarrow  \sum_{i=1}^{k} d_{\mathrm{TV}}(\PP_{f^k}^{\pi^i}, \PP_{f^*}^{\pi^i})\le \tilde{\cO}(\mathrm{poly}\log (k)(\sqrt{d\beta k}+\sqrt{d})),\ \ \forall \ k \in [K],
    \end{align}
    where $d = S^4A^2O^2H^6 \cdot \alpha ^{-4} $. The proof is provided at  step E.$1$, step E.$2$ and E.$3$ of Theorem 24 in \cite{liu2022partially}. 
\end{example}
\begin{example}[Q-type SAIL condition \cite{liu2022optimistic}]
    Q-type SAIL condition provided in \cite{liu2022optimistic} is satisfied by many RL models such as witness condition, factor MDPs and sparse linear bandits. A model class $\cF$ satisfies the Q-type $(d,c,B)$-SAIL condition if there exists two sets of mapping functions $\{p_{h,i}:\cF\to \RR^{d_{\cF}}\}_{(h,i) \times [H]\times[m]}$ and $\{q_{h,i}:\cF\to \RR^{d_{\cF}}\}_{(h,i) \times [H]\times[n]}$ such that for $f, f' \in \cF$ with optimal policy $\pi^f, \pi^{f'}$, we have 
    \begin{gather*}
        d_{\mathrm{TV}}(\PP^{\pi^f}_{f'}, \PP^{\pi^f}_{f^*})\ge c^{-1} \sum_{h=1}^H \sum_{i=1}^m \sum_{j=1}^n |\langle p_{h,i}(f) , q_{h,i}(f')\rangle|\\
        d_{\mathrm{TV}}(\PP^{\pi^f}_{f}, \PP^{\pi^f}_{f^*})\le \sum_{h=1}^H \sum_{i=1}^m \sum_{j=1}^n |\langle p_{h,i}(f) , q_{h,i}(f)\rangle|\\
        \left(\sum_{i=1}^m \Vert p_{h,i}(f)\Vert_1\right)\cdot \left(\sum_{j=1}^n \Vert q_{h,i}(f')\Vert_\infty\right)\le B.
    \end{gather*}
    From Lemma 6.3 in \cite{liu2022optimistic}, the Q-type SAIL condition also satisfies the Eq.~\eqref{l1typecondition} if we choose $d = \mbox{poly}(H)\cdot \max\{c^2,B^2\}\cdot d_\cF^2$.
\end{example}

\subsection{Rare Policy Switch Algorithm for \texorpdfstring{$\ell_1$}{}-type EC Class}\label{appendix:algorithm for l1}

In this subsection, we provide an algorithm for the $\ell_1$-type EC class with the particular loss function Eq.\eqref{l1:loss}. We only consider the MDP model for $\ell_1$-type EC class, and leave the zero-sum MG or multi-player general-sum MG as the future work. Our algorithm achieves a logarithmic switching cost while still maintaining a $\widetilde{\cO}(\sqrt{K})$ regret. The pseudo-code of the algorithm is in Algorithm \ref{alg:ET-Rare switch l1}.

In Algorithm \ref{alg:ET-Rare switch l1}, the discrepancy function $L$ is selected as the negative log-likelihood function $$L^{1:k-1}(D^{1:k-1},f) = -\sum_{i=1}^{k-1}\log \PP_{f}(\tau_H^i),$$ then the Line~\ref{line:l1confidenceset} in Algorithm~\ref{alg:ET-Rare switch l1} is equivalent  to the OMLE algorithm \cite{liu2022optimistic}. Unlike OMLE, we change the policy only when the TV distance between $f^k$ and estimated optimal policy $g^k = \inf_{g \in \cF}L^{1:k}(D_{1:k},g)$ is relatively large.  Intuitively, this distance measures the possible improvement based on the historical data. Only when we can get enough new information from the data, we recompute the confidence set and switch the policy. 


\begin{algorithm}[H]
    \begin{algorithmic}[1]
        
	\caption{Modified $\ell_1$ ABC-Rare switch}
	\label{alg:ET-Rare switch l1}
	\STATE {\textbf{Input}} $D=\emptyset,\mathscr{B}_1 = \cF$, constant $c$ in Lemma~\ref{lemma:cbetak}.
	
	\FOR{$k=1,2,\cdots,K$}
	    \STATE Compute $\pi^k = \pi_{f^k}$, where $f^k = \arg\max_{f \in \mathscr{B}^{k-1}}V_{f}^{\pi_f}(s_1)$.
	    
	    \STATE Execute policy $\pi^k$ to collect $\tau^k$, update $D = D\cup\{\tau_H\}$.

        \STATE Calculate $g^k = \inf_{g \in \cF}L^{1:k}(D^{1:k},g)$.
            
	    \IF {
        $\sum_{i=1}^{k} d_{\mathrm{TV}}(\PP_{f^k}^{\pi_i}, \PP_{g^k}^{\pi_i})\le 5c\sqrt{\beta k}$
     } \label{alg:l1beginif}
     
        \STATE Update \begin{align*}
	        \mathscr{B}^{k}=\left\{f \in \cF: L^{1:k}(D^{1:k}, f)-L^{1:k}(D^{1:k}, g_h^k)\le \beta\right\}.
        \end{align*}\label{line:l1confidenceset}
            \ELSE \STATE $\mathscr{B}^{k}=\mathscr{B}^{k-1}$.	    \ENDIF\label{alg:l1endif}
	\ENDFOR
 \end{algorithmic}
\end{algorithm}

Now we state our results for $\ell_1$-type EC class under both the rare policy switch problem and the batch learning problem.
\begin{theorem}\label{thm:l1}
    Given the hypothesis class $\cF$, we choose $\eta_h = \tau_H, \zeta_h = \emptyset$ and the loss function $\ell_{h,f'}(\zeta_h,\eta_h,f,g) = \ell_{h,f'}(\tau_H,f,g)= |\PP_f(\tau_H)/\PP_{f^*}(\tau_H)-1| $. Denote $\cB_\cF(\rho)$ as the $\rho$-bracketing number for hypothesis class $\cF$ that defined in \S \ref{appendix:define covering number}. 
    By setting $\beta = c\log(TB_{\cF}(1/K)/\delta)\ge 1$, in which  with probability at least $1-\delta$ the Algorithm~\ref{alg:ET-Rare switch} will achieve sublinear regret
    \begin{align*}
        R(K) =\tilde{\cO}(H\kappa\sqrt{d\beta K}\cdot \mathrm{poly}\log (K))
    \end{align*}
    with switch cost
    \begin{align*}
        N_{\mbox{switch}}(K)=\cO\left(\sqrt{d}\cdot\mathrm{poly}\log (K)\right).
    \end{align*}
\end{theorem}

\begin{theorem}\label{thm:batchl1}
Under the same condition as \ref{thm:l1}, if we choose the position of batches as $[k_j, k_{j+1})$ with $k_j = j\cdot \lfloor K/B\rfloor + 1$, then with probability at least $1-\delta$ we can get the following regret
    \begin{align*}
        R(K) = \widetilde{\cO}\left(\mathrm{poly}\log (K)\left(\sqrt{d}\cdot \frac{K}{B}+\sqrt{d\beta K}\right)\right).
    \end{align*}
\end{theorem}
By applying Theorem \ref{thm:l1} and Theorem \ref{thm:batchl1} to the examples in Section \ref{appendix:l1 EC class}, we can get a $\widetilde{\cO}(\sqrt{K})$ regret and a logarithmic switching cost in the rare policy switch problem, and about $\widetilde{\cO}(\sqrt{d}K/B+\sqrt{dK})$ regret in the batch learning problem for the examples of $\ell_1$-type EC class such as undercomplete POMDP and SAIL condition, where $d$ is the  parameter that is specific to the concrete examples.

\section{Proof of the Rare Policy Switch Problem} \label{sec:proof}

\subsection{Proof of Theorem~\ref{thm:l2}}
\label{sec:proofl2}
First, by choosing $\beta$ the same as Theorem~\ref{thm:l2}, we provide the following lemma, which shows that $L_h^{1:k}(D_h^{1:k}, f, \cT(f))$ is close to the optimal value $\inf_{g \in \cG}L_h^{1:k}(D_h^{1:k}, f, g)$.
\begin{lemma}\label{lemma:optimal -beta}
    For any $f \in \cF$, let $\ell_{h,f^i}(  \zeta_h, \eta_h^i,f,g)$ be a DLF, then with probability at least $1-\delta$, we have
    \begin{align}\label{eq:optimal -beta}
       0\ge \inf_{g \in \cG}L_h^{a:b}(D_h^{a:b}, f, g) -L_h^{a:b}(D_h^{a:b}, f, \cT(f)) \ge -\beta
    \end{align}
    for all $1\le a\le b \le K.$ Moreover, by choosing $f = f^*$ in Eq.\eqref{eq:optimal -beta}, we can get $f^* \in \cB^k$ for all $k \in [K]$.
\end{lemma}

Now we provide two lemmas to show that  $L_h^{1:k-1}(D^{1:k-1}_h, f^k, f^k)-L_{h}^{1:k-1}(D^{1:k-1}_h,f^k, \cT(f^k))$ is an estimate of $$\sum_{i=1}^{k-1}\Big\Vert\EE_{\zeta_h}\Big[\ell_{h,f^i}(\zeta_h, \eta_h^i, f^k,f^k)\Big]\Big\Vert_2^2.$$
\begin{lemma}\label{lemma:Cbeta to C+1beta}
    If 
    \begin{align}\label{eq:Cbeta}
L_h^{1:k-1}(D^{1:k-1}_h, f^k, f^k)-L_{h}^{1:k-1}(D^{1:k-1}_h,f^k, \cT(f^k))\le C\beta
\end{align}
for some constant $100\ge C\ge 1$, then with probability at least $1-2\delta$,
\begin{align}
\sum_{i=1}^{k-1}\Big\Vert\EE_{\zeta_h}\Big[\ell_{h,f^i}(\zeta_h, \eta_h^i, f^k,f^k)\Big]\Big\Vert_2^2\le (C+1)\beta.\label{eq:first lemma a2}
\end{align}
Moreover, we have 
\begin{align}\label{eq:C+1beta second inequality}
\sum_{i=1}^{k-1}\EE_{\eta_h\sim \pi^i}\Big\Vert\EE_{\zeta_h}\Big[\ell_{h,f^i}(\zeta_h, \eta_h, f^k,f^k)\Big]\Big\Vert_2^2\le (C+1)\beta.
\end{align}
Also, if all constant $C \ge 2,$ we have

\begin{align}
\sum_{i=1}^{k-1}\Big\Vert\EE_{\zeta_h}\Big[\ell_{h,f^i}(\zeta_h, \eta_h^i, f^k,f^k)\Big]\Big\Vert_2^2\le (2C)\beta,\label{eq:first lemma a2 2C}
\end{align}
and 
\begin{align}\label{eq:2Cbeta second inequality}
\sum_{i=1}^{k-1}\EE_{\eta_h\sim \pi^i}\Big\Vert\EE_{\zeta_h}\Big[\ell_{h,f^i}(\zeta_h, \eta_h, f^k,f^k)\Big]\Big\Vert_2^2\le (2C)\beta.
\end{align}

\end{lemma}


\begin{lemma}\label{lemma:Cbeta to C-1beta}
    If we have
    \begin{align*}L_h^{1:k-1}(D^{1:k-1}_h, f^k, f^k)-L_{h}^{1:k-1}(D^{1:k-1}_h,f^k, \cT(f^k))\ge C\beta
\end{align*}
for some constant $100\ge C\ge 2$, then with probability at least $1-2\delta$
\begin{align}
\sum_{i=1}^{k-1}\Big\Vert\EE_{\zeta_h}\Big[\ell_{h,f^i}(\zeta_h, \eta_h^i, f^k,f^k)\Big]\Big\Vert_2^2\ge (C-1)\beta.\label{eq:first lemma a3}
\end{align}
Moreover, we have 
\begin{align}\label{eq:C-1beta second inequality}
\sum_{i=1}^{k-1}\EE_{\eta_h\sim \pi^i}\Big\Vert\EE_{\zeta_h}\Big[\ell_{h,f^i}(\zeta_h, \eta_h, f^k,f^k)\Big]\Big\Vert_2^2\ge (C-1)\beta.
\end{align}
Also, if all constant $C \ge 2,$ we have

\begin{align}
\sum_{i=1}^{k-1}\Big\Vert\EE_{\zeta_h}\Big[\ell_{h,f^i}(\zeta_h, \eta_h^i, f^k,f^k)\Big]\Big\Vert_2^2\ge (C/2)\beta,\label{eq:first lemma a2 C/2}
\end{align}
and 
\begin{align}\label{eq:C/2beta second inequality}
\sum_{i=1}^{k-1}\EE_{\eta_h\sim \pi^i}\Big\Vert\EE_{\zeta_h}\Big[\ell_{h,f^i}(\zeta_h, \eta_h, f^k,f^k)\Big]\Big\Vert_2^2\ge (C/2)\beta.
\end{align}
\end{lemma}
Combining Lemma~\ref{lemma:Cbeta to C+1beta} and Lemma~\ref{lemma:Cbeta to C-1beta}, we can claim that the 
 term $L_h^{1:k-1}(D^{1:k-1}_h, f^k, f^k)-L_{h}^{1:k-1}(D^{1:k-1}_h,f^k, \cT(f^k))$ for $h \in [H]$ in Algorithm \ref{alg:ET-Rare switch l1} is a good estimate for the expectation of loss function.
\paragraph{Proof of Regret}
First, we claim that for each episode $k \in [K]$, 
\begin{align}\label{eq:change7beta}
L_h^{1:k-1}(D^{1:k-1}_h, f^k, f^k)-L_h^{1:k-1}(D^{1:k-1}_h,f^k, \cT(f^k))\le 6\beta.
\end{align}
If the policy changes at episode $k-1$, $L_h^{1:k-1}(D^{1:k-1}_h, f^k, f^k)-\inf_{g \in \cG}L_h^{1:k-1}(D^{1:k-1}_h,f^k, g)\le \beta$ for the construction of confidence set. Combining with Eq.~\eqref{eq:optimal -beta} we can get Eq.~\eqref{eq:change7beta}.  If the policy has not been changed and $f^{k-1}=f^k$,  $L_h^{1:k-1}(D^{1:k-1}_h, f^{k-1}, f^{k-1})-\inf_{g \in \cG}L_h^{1:k-1}(D^{1:k-1}_h,f^{k-1}, g)\le 5\beta$. Combining with Eq.~\eqref{eq:optimal -beta}, we can get 
\$
L_h^{1:k-1}(D^{1:k-1}_h, f^{k-1}, f^{k-1})-L_h^{1:k-1}(D^{1:k-1}_h,f^{k-1}, \cT(f^{k-1}))\le 6\beta.
\$
Then Eq.~\eqref{eq:change7beta} can be derived by the fact $f^{k-1}=f^k$.
Now based on Lemma~\ref{lemma:Cbeta to C+1beta} and Eq.~\eqref{eq:change7beta}, we have 
\begin{align}\label{eq:eluderconditionl2}
    \sum_{i=1}^{k-1}\Big\Vert\EE_{\zeta_h}\Big[\ell_{h,f^i}(\zeta_h, \eta_h^i, f^k,f^k)\Big]\Big\Vert_2^2\le 7\beta.
\end{align}
Now by the $\ell_2$-type eluder condition and Cauchy's inequality, we have 
\begin{align}\label{eq:eluderresult}
    \sum_{i=1}^{k} \Big\Vert\mathbb{E}_{\zeta_h}\Big[\ell_{h,f^i}(  \zeta_h, \eta_h^i, f^i,f^i)  \Big]\Big\Vert_2\le \cO(\sqrt{d\beta k}\cdot \log k).
\end{align}
Also, by the dominance property, 
\begin{align}
    \sum_{i=1}^k (V_{1,f^*}(s_1)-V^{\pi_i}(s_1))&\le \sum_{i=1}^k (V_{1,f^i}(s_1)-V^{\pi_i}(s_1))\nonumber\\
        &\le \kappa\sum_{h=1}^H \sum_{i=1}^{k}  \EE_{\eta_h}\Big\Vert\mathbb{E}_{\zeta_h}\Big[\ell_{h,f^i}(  \zeta_h, \eta_h, f^i,f^i)  \Big]\Big\Vert_2\nonumber\\
        &=\kappa\cdot \sum_{h=1}^H \left(\sum_{i=1}^{k}  \Big\Vert\mathbb{E}_{\zeta_h}\Big[\ell_{h,f^i}(  \zeta_h, \eta_h^i, f^i,f^i)  \Big]\Big\Vert_2 + \widetilde{\cO}\left(\sqrt{K}\log K\right)\right)\label{eq:azuma-hoeff bound}\\
        &=\widetilde{\cO}(\kappa H\sqrt{d\beta K}\cdot \log K),\nonumber
\end{align}
where the first inequality is derived from Lemma \ref{lemma:optimal -beta} and $\cT(f^*) = f^*$, which implies $f^* \in \rB^k$ for all $k \in [K]$ by 
    $L_h^{1:k}(D_h^{1:k}, f^*, f^*)-\inf_{g \in \cG}L_h^{1:k}(D_h^{1:k}, f^*, g)\le \beta.$
Eq.~\eqref{eq:azuma-hoeff bound} is derived from the Azuma-Hoeffding's inequality and the boundness property of the loss function $\ell$.

\paragraph{Proof of Switch Cost}

Fixed a step $h \in [H]$,
assume the policy changes at episode $b_1^h, b_2^h,\cdots,b_l^h$ because the in-sample error at step $h$ is larger than the threshold, \$L_h^{1:k}(D^{1:k}_h, f^k, f^k)-\inf_{g \in \cG}L_{h}^{1:k}(D^{1:k}_h,f^k, g)\ge 5\beta,\$ where $l$ is the number of the policy switch because the error at step $h$ is larger than the threshold $5\beta$. Then by Lemma~\ref{lemma:optimal -beta}, we have 
\begin{align}
    L_h^{1:k}(D^{1:k}_h, f^k, f^k)-L_{h}^{1:k}(D^{1:k}_h,f^k, \cT(f^k))\ge 4\beta\label{eq:bih greater 2betanew}
\end{align}
for all $k = b_i^h, 1\le i\le l$. 
Define $b_0^h=0$ for simplicity.
Fixed an $1\le j\le l$ and consider the latest time $b'$ that changes the policy before episode $b_j^h$, we will get $b'\ge b_{j-1}^h$ and
\begin{align}
    L_h^{1:b'}(D^{1:b'}_h, f^{b'+1}, f^{b'+1})-\inf_{g \in \cG}L_{h}^{1:b'}(D^{1:b'}_h,f^{b'+1}, g)\le \beta,\nonumber\\
    L_h^{1:b'}(D^{1:b'}_h, f^{b'+1}, f^{b'+1})-L_{h}^{1:b'}(D^{1:b'}_h,f^{b'+1}, \cT(f^{b'+1}))\le \beta.\label{eq:b'lessthan 2betanew}
\end{align}

Since at episode $b'+1,\cdots,b_j^h-1$ the confidence set is not changed, we have $\rB^{b'} = \rB^{b'+1} = \cdots = \rB^{b_j^h-1}$ and $f^{b'+1} = f^{b'+2} = \cdots = f^{b_j^h}$. Then combining Eq.~\eqref{eq:bih greater 2betanew} and Eq.~\eqref{eq:b'lessthan 2betanew}, we can get
\begin{align*}
    L^{b'+1:b_j^h}_h(D_h^{b'+1:b_j^h}, f^{b'+1}, f^{b'+1})-L^{b'+1:b_j^h}_h(D_h^{b'+1:b_j^h}, f^{b'+1}, \cT(f^{b'+1}))\ge 3\beta.
\end{align*}
By Lemma~\ref{lemma:Cbeta to C-1beta}, with probability at least $1-\delta$,  $\sum_{i=b'+1}^{b_j^h}\Vert \EE_{\zeta_h}[\ell_{h,f^i}(  \zeta_h, \eta_h^i, f^i,f^i)] \Vert_2^2 \ge 2\beta$. By $b'\ge b_{j-1}^h$, we can see that 

\begin{align*}
    \sum_{i=b_{j-1}^h+1}^{b_j^h}   \Vert \EE_{\zeta_h}[\ell_{h,f^i}(  \zeta_h, \eta_h^i, f^i,f^i)] \Vert_2^2 \ge 2\beta.
\end{align*}
Now sum over all $1\le i\le l$, we can get 
\begin{align}\label{eq:greatthan2lbetanew}
    \sum_{i=1}^{K}   \Big\Vert \EE_{\zeta_h}\Big[\ell_{h,f^i}(  \zeta_h, \eta_h^i, f^i,f^i)\Big] \Big\Vert_2^2\ge  \sum_{j=1}^{l-1}\sum_{i = b_{j-1}^h+1}^{b_j^h}   \Big\Vert \EE_{\zeta_h}\Big[\ell_{h,f^i}(  \zeta_h, \eta_h^i, f^i,f^i)\Big] \Big\Vert_2^2\ge 2(l-1)\beta,
\end{align}
where $l$ is the number of switches corresponding to step $h \in [H]$.

Now by Eq.~\eqref{eq:eluderconditionl2} and $\ell_2$-type eluder condition, \begin{align}\sum_{i=1}^{K}   \Big\Vert \EE_{\zeta_h}\Big[\ell_{h,f^i}(  \zeta_h, \eta_h^i, f^i,f^i)\Big] \Big\Vert_2^2 \le \cO(d\beta \log K).\nonumber\end{align} 
Combining with Eq.~\eqref{eq:greatthan2lbetanew}, $l=\cO(d \log K) $ and the total  switching cost can be bounded by $\cO(dH\log K)$.

\subsection{Proof of Theorem~\ref{thm:l1}}\label{appendix:proof of l1}
Since $L^{1:k-1}(D^{1:k-1},g) = -\sum_{i=1}^{k-1}\log \PP_g(\tau_H^i)$, we first show that $g^k = \arg\inf_g L^{1:k-1}(D^{1:k-1},g)$ is closed to $f^*$ with respect to TV distance. Since we choose $\beta \ge 1$, we simplify the right side of Eq.\eqref{eq:l1-type condition} as $\widetilde{\cO}(\sqrt{d\beta k}\cdot (\log k)^2)$.
\begin{lemma}\label{lemma:cbetak}
   For all $k \in [K]$, let $g_k = \arg\max_g L^{1:k}(D^{1:k},g)$, then with probability at least $1-\delta$, \begin{align}\sum_{i=1}^{k} d_{\mathrm{TV}}(\PP^{\pi^i}_{f^*}, \PP^{\pi^i}_{g^k}) \le c\sqrt{\beta k}.\nonumber\end{align}
\end{lemma}
\begin{proof}
    By proposition 14 in \cite{liu2022partially}, there is a constant $c>0$ such that
    \begin{align}
        \sum_{i=1}^{k} d_{\mathrm{TV}}^2(\PP^{\pi^i}_{f^*}, \PP^{\pi^i}_{g^k}) \le c^2\beta\nonumber
    \end{align}
    with probability at least $1-\delta$ for all $k \in [K]$.
    Then by Cauchy's inequality, we have 
    \begin{align}
        \sum_{i=1}^{k} d_{\mathrm{TV}}(\PP^{\pi^i}_{f^*}, \PP^{\pi^i}_{g^k}) \le c\sqrt{\beta k}\nonumber
    \end{align}
    for all $k \in [K]$.
\end{proof}
 We then prove that \begin{align} \label{eq:TV6betak}\sum_{i=1}^{k-1} d_{\mathrm{TV}}(\PP^{\pi^i}_{f^{k}}, \PP^{\pi^i}_{f^*}) \le 6c\sqrt{\beta k}.\end{align} 
\paragraph{Case 1:} If at episode $k-1$ the confidence set are not changed, it implies that 
\begin{align}
    \sum_{i=1}^{k-1} d_{\mathrm{TV}}(\PP^{\pi^i}_{f^{k-1}}, \PP^{\pi^i}_{g^{k-1}}) \le 5c\sqrt{\beta (k-1)}.\nonumber
\end{align}
Thus since $f^{k-1}=f^k$,
\begin{align}
    \sum_{i=1}^{k-1} d_{\mathrm{TV}}(\PP^{\pi^i}_{f^{k}}, \PP^{\pi^i}_{g^{k-1}}) \le 5c\sqrt{\beta k}.\nonumber
\end{align}
Combining with Lemma~\ref{lemma:cbetak}, we can get Eq.~\eqref{eq:TV6betak}

\paragraph{Case 2:} If the confidence set are changed at episode $k-1$, then 
\begin{align}
    \sum_{i=1}^{k-1} d_{\mathrm{TV}}(\PP^{\pi^i}_{f^{k}}, \PP^{\pi^i}_{g^{k-1}}) \le c\sqrt{\beta k}.\nonumber
\end{align}
Combining with Lemma~\ref{lemma:cbetak},
we can get Eq.~\eqref{eq:TV6betak}.

Now since $   \EE_{\zeta_h}[\ell_{h,f^i}(  \zeta_h, \eta_h^i, f^i,f^i)  ] = d_{\mathrm{TV}}(\PP_{f^i}^{\pi^i}, \PP_{f^*}^{\pi^i})$, by $\ell_1$-type eluder condition with Eq.~\eqref{eq:TV6betak}, there is a constant $c'$ such that
\begin{align}\label{eq:l1eluderresult}
    \sum_{i=1}^k d_{\mathrm{TV}}(\PP^{\pi^i}_{f^{i}}, \PP^{\pi^i}_{f^*})\le c'(\sqrt{d\beta k}\cdot\mathrm{poly}\log (k)).
\end{align}
Now since from Proposition 13 in \cite{liu2022partially}, by choosing $\beta = \cO(\log(KN_{\cF}(1/K)/\delta))$ with a sufficiently large constant, we know $f^* \in \rB^k$. Then
\begin{align*}
    R(K) &= \sum_{i=1}^K(V_{1,f^*}^{\pi^*}(s_1)-V_{1,f^*}^{\pi^i}(s_1))\\&\le \sum_{i=1}^K(V_{1,f^i}^{\pi^i}(s_1)-V_{1,f^*}^{\pi^i}(s_1))\\&\le 
    \sum_{i=1}^KH\cdot d_{\mathrm{TV}}(\PP_{f^i}^{\pi^i}, \PP_{f^*}^{\pi^i})\\&\le 
    c'H(\sqrt{d\beta k}\cdot\mathrm{poly}\log (k)),
\end{align*}
where the first inequality holds because $f^* \in \cB^k$ for all $k \in [K],$ and $f^k$ is the optimal policy within the confidence set $\cB^k.$

\paragraph{Switch Cost}
Now assume the policy changed at time $b_1,\cdots,b_l$ and define $b_0 = 0$, then \begin{align}\sum_{i=1}^{b_j} d_{\mathrm{TV}}(\PP_{f^{b_j}}^{\pi_i}, \PP_{g^{b_j}}^{\pi_i})\ge 5c\sqrt{\beta b_j}\nonumber\end{align} for all $1\le j\le l$. By the triangle inequality, we can get
\begin{align}
    \sum_{i=1}^{b_j} d_{\mathrm{TV}}(\PP_{f^{b_j}}^{\pi_i}, \PP_{f^*}^{\pi_i})\ge\sum_{i=1}^{b_j} (d_{\mathrm{TV}}(\PP_{f^{b_j}}^{\pi_i}, \PP_{g^{b_j}}^{\pi_i}) -d_{\mathrm{TV}}(\PP_{f^*}^{\pi_i}, \PP_{g^{b_j}}^{\pi_i}))\ge4c\sqrt{\beta b_j}\label{eq:l1greatthan4cbeta}
\end{align}
for all $1\le j\le l$. Now by the construction of confidence set $\rB^k$, we have $f^{b_{j-1}+1} = \cdots = f^{b_j}$ and
\begin{align}\label{eq:bj-1+1}
    \sum_{i=1}^{b_{j-1}} d_{\mathrm{TV}}(\PP_{f^{b_{j-1}+1}}^{\pi_i}, \PP_{f^*}^{\pi_i})\le c\sqrt{\beta (b_{j-1})}.
\end{align}
Thus combining with Eq.~\eqref{eq:l1greatthan4cbeta} and Eq.~\eqref{eq:bj-1+1}, 
\begin{align}\label{eq:3c}
    \sum_{i=b_{j-1}+1}^{b_{j}} d_{\mathrm{TV}}(\PP_{f^{i}}^{\pi_i}, \PP_{f^*}^{\pi_i})=\sum_{i=b_{j-1}+1}^{b_{j}} d_{\mathrm{TV}}(\PP_{f^{b_{j}}}^{\pi_i}, \PP_{f^*}^{\pi_i})\ge 3c\sqrt{\beta b_j},
\end{align}
and for all $k \in [K]$,
\begin{align*}
    \sum_{i=1}^k d_{\mathrm{TV}}(\PP_{f^{i}}^{\pi_i}, \PP_{f^*}^{\pi_i}) \ge \sum_{b_j\le k}\left(\sum_{i=b_{j-1}+1}^{b_{j}} d_{\mathrm{TV}}(\PP_{f^{i}}^{\pi_i}, \PP_{f^*}^{\pi_i})\right)\ge 3c\sum_{b_j\le k}\sqrt{\beta b_j}.
\end{align*}
The first inequality is because we divide the time interval $[1,k]$ to some intervals $[b_{j-1}+1, b_j]$ for all $b_j\le k$, and the second inequality is from Eq.~\eqref{eq:3c}. Now fixed a $k \in [K]$, by Eq.~\eqref{eq:l1eluderresult}, we have 
\begin{align*}
    3c\sum_{b_j\le k}\sqrt{\beta b_j}\le c'(\sqrt{d\beta k}\cdot\mathrm{poly}\log (k)).
\end{align*}
Denote the number of $j$ such that $b_j \in (k/2,k]$ are $s_k$, i.e. $s_k = |\{j:b_j \in (k/2,k]\}|$, then 
\begin{align}
    3cs_k\cdot \sqrt{\beta k/2}&\le 3c\sum_{b_j\le k}\sqrt{\beta b_j}\le c'(\sqrt{d\beta k}\cdot\mathrm{poly}\log (k)),\nonumber\\
    s_k&\le \frac{c'\sqrt{2}}{3c}(\sqrt{d}\cdot\mathrm{poly}\log (K)).\label{eq:supperbound }
\end{align}
Now we divide the interval $[1,K]$ into $(\lceil K/2\rceil, K],(\lceil K/4\rceil, \lfloor K/2\rfloor],\cdots,(\lceil \frac{K}{2^m}\rceil, \lfloor \frac{K}{2^{m-1}}\rfloor], [1]$ with $\lceil K/2^{m}\rceil =1$,  and $m = \cO(\log K)$. Then the number of $b_j$ in each interval is upper bounded by $\cO(\sqrt{d}(\log K)^2)$ from Eq.~\eqref{eq:supperbound }, because $\frac{c'\sqrt{2}}{3c}$ does not depend on the selection of $k$. Then the total number of policy switch is upper bounded by $\cO(\sqrt{d}\cdot\mathrm{poly}\log (K))$.

\section{Proof of the Batch Learning Problem}\label{appendix:proof of batch}

\subsection{Proof of Theorem~\ref{thm:batchl2}}
\begin{proof}
    We first fix an $h \in [H]$ in the proof.
    Since we change our policy at time $k_j = j\cdot \lfloor K/B \rfloor +1$ for $j\ge 0$, we can know that 
    \begin{align*}
&L_h^{1:k_j-1}(D_{h}^{1:k_j-1}, f^{k_j}, f^{k_j}) - L_h^{1:k_j-1}(D_h^{1:k_j-1}, f^{k_j}, \cT(f^{k_j}))\\&\quad\le L_h^{1:k_j-1}(D_{h}^{1:k_j-1}, f^{k_j}, f^{k_j}) - \inf_{g \in \cG} L_h^{1:k_j-1}(D_h^{1:k_j-1}, f^{k_j}, g)\\&\quad\le \beta.
    \end{align*}
    We denote 
    \begin{align*}
        c_j :=\max_{k \in [k_j:k_{j+1}-1]}\left(L_h^{k_j: k}(D_h^{k_j:k}, f^{k_j}, f^{k_j})-L_h^{k_j: k}(D_h^{k_j:k}, f^{k_j}, \cT(f^{k_j}))\right)/\beta.
    \end{align*}
    The parameter $c_j$ represents the maximum fitting error for data of this batch and the model $f^{k_j}$ determined by previous batches. If the error is small, the regret can be easily bounded. Thus we only need to prove that the number of batches with large in-sample error is small.
    Denote $S = \{j \ge 0\mid c_j> 5\}$ are all "Bad" batches with relatively large in-sample error, then we  prove that $|S|\le \widetilde{\cO}(d(\log K)^2)$.

    We will prove the following lemma to upper bound $|S|.$
    \begin{lemma}\label{lemma:from C/2 to C}
        For a fixed $C\ge 10$, with probability at least $1-\delta$, we will have 
        \begin{align*}
            \left|\left\{j\in S\ \Bigg|\ \frac{C}{2}\le c_j\le C\right\}\right|\le \widetilde{\cO}(d\log K).
        \end{align*}
    \end{lemma}
    \begin{proof}
        Denote $\left\{j\in S\mid\frac{C}{2}\le c_j\le C\right\} = \{i_1,\cdots, i_M\}$ with $M=  \left|\left\{j\in S\mid\frac{C}{2}\le c_j\le C\right\}\right|$ and $i_1\le i_2\le \cdots \le i_M$.
         Then for $m \in [M]$ and $k \in [k_{i_m}, k_{i_{m}+1}-1]$, we have 
         \begin{align*}
        L_h^{1:k-1}(D_{h}^{1:k-1}, f^{k}, f^{k}) - L_h^{1:k-1}(D_h^{1:k-1}, f^{k}, \cT(f^{k}))\le (1+c_{i_m})\beta \le (1+C)\beta.
    \end{align*}
    Then by Lemma~\ref{lemma:Cbeta to C+1beta}, we can get 
    \begin{align*}
         \sum_{i=1}^{k-1}\Big\Vert \EE_{\zeta_h}\Big[\ell_{h,f^i}(\zeta_h, \eta_h^i,f^k,f^k)\Big]\Big\Vert_2^2\le \widetilde{\cO}((2+2C)\beta)
    \end{align*}
    for any $m \in [M]$ and $k \in [k_{i_m}, k_{i_m+1}-1]$.
    Then we will have 
    \begin{align*}
        \sum_{\substack{1\le i\le k-1 \\ i\in [k_{i_m}, k_{i_m+1}-1], m \in [M]}}\Big\Vert\EE_{\zeta_h}\Big[\ell_{h,f^i}(\zeta_h, \eta_h^i,f^k,f^k)\Big]\Big\Vert_2^2 & \le \sum_{1\le i\le k-1}\Big\Vert \EE_{\zeta_h}\Big[\ell_{h,f^i}(\zeta_h, \eta_h^i,f^k,f^k)\Big]\Big\Vert_2^2\\&\le \widetilde{\cO}(2+2C)\beta.
    \end{align*}
    By using the $\ell_2$-type eluder condition for all such $i \in [k_{i_m}, k_{i_m+1}-1]$, we can have 
    \begin{align}\label{eq:upperbound d(2+C)}
        \sum_{\substack{1\le i\le k \\ i\in [k_{i_m}, k_{i_m+1}-1], m \in [M]}}\Big\Vert \EE_{\zeta_h}\Big[\ell_{h,f^i}(\zeta_h, \eta_h^i, f^i,f^i)\Big]\Big\Vert_2^2 \le \widetilde{\cO}(d(2+2C)\beta\log K).
    \end{align}
    Also, by Lemma~\ref{lemma:Cbeta to C-1beta} and the fact that $c_{i_m}\ge C/2$, with probability at least $1-\delta$, for any $m \in [M]$,
    \begin{align}\label{eq:lowerbound C/2-1}
    &\sum_{i=k_{i_m}}^{k_{i_m+1}-1} \Big\Vert \EE_{\zeta_h}\Big[\ell_{h,f^i}(\zeta_h, \eta_h^i, f^i, f^i)\Big]\Big\Vert_2^2  
    \\&\quad= \sum_{i=k_{i_m}}^{k_{i_m+1}-1}\Big\Vert \EE_{\zeta_h}\Big[\ell_{h,f^i}(\zeta_h, \eta_h^i, f^{k_{i_m}},f^{k_{i_m}})\Big]\Big\Vert_2^2
    \\&\quad=\max_{k \in [k_{i_m}, k_{i_m+1}-1]}  \sum_{i=k_{i_m}}^{k}\Big\Vert \EE_{\zeta_h}\Big[\ell_{h,f^i}(\zeta_h, \eta_h^i, f^{k_{i_m}},f^{k_{i_m}})\Big]\Big\Vert_2^2
    \\&\quad\ge \max_{k \in [k_{i_m}, k_{i_{m+1}}-1]} \Big(
        L_h^{k_{i_m}:k}(D_h^{k_{i_m}:k}, f^{k_{i_m}}, f^{k_{i_m}})\nonumber\\&\qquad\quad-L_h^{k_{i_m}:k}(D_h^{k_{i_m}:k}, f^{k_{i_m}}, \cT(f^{k_{i_m}}))-\beta\Big)\\&\quad\ge
        (C/4)\beta.
    \end{align}
    Hence from \eqref{eq:upperbound d(2+C)}, \eqref{eq:lowerbound C/2-1} and $C\ge 10$, we can see that 
    \begin{align*}
        M = \widetilde{\cO}(d\log K).
    \end{align*}
    \end{proof}
    Now by Lemma~\ref{lemma:from C/2 to C}, we can divide $S = \{j\ge 0\mid c_j>5\}$ as $S^{(1)}, S^{(2)},\cdots$ that $S^{(i)} = \{j\ge 0\mid 5\cdot 2^{i-1}\le c_j\le 5\cdot 2^i\}$. Then for each $i$, $|S^{(i)}|\le \widetilde{\cO}(d\log K)$.
    Since we have a trivial upper bound $c_i\le (K/B)$ for all $0\le i<B$, we know the number of sets $S^{(i)}$ is at most $\log_2 (K/B)$, then $|S| \le \widetilde{\cO}(d(\log K)^2)$.

    Now, since for any $j\notin S$ and $k\in [k_j, k_{j+1}-1]$, we have 
    \begin{align*}
        &L_h^{1:k-1}(D_{h}^{1:k-1}, f^{k}, f^{k}) - L_h^{1:k-1}(D_h^{1:k-1}, f^{k}, \cT(f^{k})) \\&\quad = (L_h^{1:k_j-1}(D_h^{1:k_j-1}, f^{k_j}, f^{k_j}) + L_h^{k_j:k-1}(D_h^{1:k-1}, f^k, f^k)) \\&\qquad- (L_h^{1:k_j-1}(D_h^{1:k_j-1}, f^{k_j}, \cT(f^{k_j}) + L_h^{k_j:k-1}(D_h^{1:k-1}, f^k, \cT(f^k)))\\
        &\quad\le (1+c_j)\beta\\
        &\quad\le 6\beta.
    \end{align*}
    By Lemma~\ref{lemma:Cbeta to C+1beta},  we can get 
    \begin{align*}
        \sum_{\substack{1\le i\le k-1 \\ i \in [k_j, k_{j+1}-1], j \notin S}} \Big\Vert \EE_{\zeta_h}\Big[\ell_{h,f^i}(\zeta_h, \eta_h^i, f^k, f^k)\Big]\Big\Vert_2^2&\le \sum_{1\le i\le k-1}\Big\Vert \EE_{\zeta_h}\Big[\ell_{h,f^i}(\zeta_h, \eta_h^i, f^k, f^k)\Big]\Big\Vert_2^2 \le 7\beta.
    \end{align*} 
    By the $\ell_2$-type eluder condition, the regret caused by "Good" batches can be upper bounded by  
    \begin{align}
        \sum_{\substack{1\le i\le k\\ i \in [k_j, k_{j+1}-1], j \notin S}} \Big\Vert \EE_{\zeta_h}\Big[\ell_{h,f^i}(\zeta_h, \eta_h^i, f^i, f^i)\Big]\Big\Vert_2^2\le \widetilde{\cO}(d\beta \log K).\label{ineq:good batch}
    \end{align} 
    
    Then by the dominance property and Azuma-Hoeffding's inequality, we have 
    
    \begin{align}
        &\sum_{i=1}^K(V_{1,f^i}(s_1)-V_1^{\pi_i}(s_1))\nonumber\\&\quad\le \kappa \left(\sum_{h=1}^H \sum_{i=1}^k \Big\Vert \EE_{\zeta_h}\Big[\ell_{h,f^i}(\zeta_h, \eta_h^i, f^i, f^i)\Big]\Big\Vert_2 + \widetilde{\cO}(\sqrt{HK}\log K)\right)\label{eq:azuma+dominance}\\&\quad\le 
        \kappa\sum_{h=1}^H\left( \sum_{\substack{1\le i\le k\\ i \in [k_j, k_{j+1}-1], j \notin S}}\Big\Vert \EE_{\zeta_h}\Big[\ell_{h,f^i}(\zeta_h, \eta_h^i, f^i, f^i)\Big]\Big\Vert_2 \right.\nonumber\\&\qquad\quad\left.+ \sum_{\substack{1\le i\le k\\ i \in [k_j, k_{j+1}-1], j \in S}}\Big\Vert \EE_{\zeta_h}\Big[\ell_{h,f^i}(\zeta_h, \eta_h^i, f^i, f^i)\Big]\Big\Vert_2\right)+ \kappa \widetilde{\cO}(\sqrt{HK}\log K)\nonumber\\
        &\quad= \kappa\sum_{h=1}^H \left(\widetilde{\cO}(\sqrt{d\beta K} \log K) + R\cdot \lfloor K/B\rfloor \widetilde{\cO}(d(\log K)^2)\right)+ \kappa \widetilde{\cO}(\sqrt{HK}\log K)\label{eq:thirdline}\\&\quad=\widetilde{\cO}\left(\kappa H\sqrt{d\beta K\log K}+ \frac{dHK}{B}(\log K)^2\right)\nonumber.
    \end{align} 
 The inequality Eq.\eqref{eq:azuma+dominance} is derived by the dominance property and Azuma-Hoeffding's inequality, and the first equality Eq.\eqref{eq:thirdline} holds by Cauchy's inequality,  Eq.\eqref{ineq:good batch}, $|S|\le \widetilde{\cO}(d(\log K)^2)$ and $\EE_{\zeta_h}[\ell_{h,f^i}(\zeta_h,\eta_h^i,f^i,f^i)]\Vert_2\le R$. 
\end{proof}
\subsection{Discussion about Adaptive Batch Setting}\label{appendix:adaptive batch}
For the adaptive batch setting, to achieve a  $O(\sqrt{K})$ regret, we want to let every step have a small in-sample error. That is, $L_h^{1:k}(D_h^{1:k}, f, f)-\inf_{g \in \cG}L_h^{1:k}(D_h^{1:k}, f, g)\le \cO(\beta)$ for all episode $k$. Then the regret can be easily bounded by $\cO(\sqrt{K})$ using previous analyses in Theorem~\ref{thm:l2}. 

To guarantee this, we modify the rare policy switch Algorithm~\ref{alg:ET-Rare switch}. Note that the Algorithm~\ref{alg:ET-Rare switch} guarantees that each step has $\cO(\beta)$ in-sample error by the updating rule. However, in the adaptive batch setting, we cannot receive the feedback of the current batch. To solve this problem, we can use a simple double trick: We observe the feedback when the length of the batch doubles, and check whether the in-sample error $L_h^{1:k}(D_h^{1:k}, f, f)-\inf_{g \in \cG}L_h^{1:k}(D_h^{1:k}, f, g)$ is greater than $5\beta$. Whenever we observe this error is greater than $5\beta$, we change our policy and begin to choose a batch with length $1$. The entire algorithm is presented in Algorithm \ref{alg:ET-adabatch}.

We show that for MDP, with this double trick, we can still maintain 
$$L_h^{1:k}(D_h^{1:k}, f, f)-\inf_{g \in \cG}L_h^{1:k}(D_h^{1:k}, f, g)\le \cO(\beta)$$ 
for each round $k \in [K]$, thus achieve the $\cO(\sqrt{K})$ regret. 
\begin{lemma}\label{lemma:adaptive batch}
    In Algorithm \ref{alg:ET-adabatch}, for any $k \in [K]$,  we have $L_h^{1:k}(D_h^{1:k}, f, f)-\inf_{g \in \cG}L_h^{1:k}(D_h^{1:k}, f, g)\le \cO(\beta)$.
\end{lemma}
It is worth noting that in Algorithm \ref{alg:ET-adabatch} we choose not to consider the zero-sum MG in the adaptive batch setting. The primary reason is that the policy $\pi^i = (\upsilon^i, \mu^i)$ can change within a batch, which makes the double trick fail to work. Indeed, this nature introduces technical difficulties when proving the Lemma \ref{lemma:adaptive batch}.

Now given that each step has a low in-sample error, we can show Algorithm \ref{alg:ET-adabatch} have $\cO(\sqrt{K})$ regret, and the number of batches is at most $\cO((\log K)^2)$. The square term arises from the extra division of batches for a double trick.

\begin{theorem}
    Under the adaptive batch setting and the same condition as Theorem~\ref{thm:l2}, with high probability at least $1-\delta$ the Algorithm \ref{alg:ET-adabatch} will achieve a sublinear regret 
    \begin{align*}
        R(K)\le \widetilde{\cO}(\kappa H \sqrt{d\beta K}\cdot \mbox{poly}(\log K)).
    \end{align*}
    Moreover, the number of batches is at most $\cO(dH\cdot \mbox{poly}(\log K)).$
\end{theorem}

\begin{proof}
By Lemma~\ref{lemma:adaptive batch}, at each episode $k \in [K]$, we have a small $\cO(\beta)$ error of the previous data. Then applying the proof of Theorem~\ref{thm:l2}, the upper bound of regret is  $\widetilde{\cO}(\kappa H \sqrt{d\beta K}\cdot \mbox{poly}(\log K)).$

Now consider the number of batches. Since we change the policy only when the error is larger than $5\beta$, the policy will be changed at most $\cO(dH\cdot \mbox{poly}(\log K))$ times by Theorem~\ref{thm:l2}. In addition, suppose the agent changes the policy at $k_1$ and $k_2$ while keeps the change unchanged at episode $k_1\le k<k_2$, the number of batches between $k_1$ and $k_2$ is at most $\cO(\log K)$. Hence the total number of batches can be upper bounded by $\cO(dH\cdot \mbox{poly}(\log K))\cdot \log K = \cO(dH\cdot \mbox{poly}(\log K)).$
\end{proof}

\begin{algorithm}[t]
    \begin{algorithmic}[1]
        
	\caption{$\ell_2$-EC-Adaptive Batch}
	\label{alg:ET-adabatch}
	\STATE {\textbf{Input}} $D_1,D_2,\cdots,D_H=\emptyset,\mathscr{B}_1 = \cF, length = 1$.
	
	\FOR{$k=1,2,\cdots,K$}
            
	    \STATE \textcolor{blue}{(MDP):} Compute $\pi^k = \pi_{f^k}$, where $f^k = \arg\max_{f \in \mathscr{B}^{k-1}}V_{f}^{\pi_f}(s_1)$.\label{line:adabatch oracle}
	    \STATE Execute policy $\pi^k$.

            \IF{$length = 2^i$ for some $i\ge 0$}
            \STATE Observe the feedback and update $D_h^{1:k} = \{\zeta_h^{1:k}, \eta_h^{1:k}\}$.
            
	    \IF {$L_h^{1:k}(D^{1:k}_h, f^k, f^k)-\inf_{g \in \cG}L_h^{1:k}(D^{1:k}_h,f^k, g)\ge 5\beta$ for some $h \in [H]$} \label{alg:adabatch l2beginif}
        
        \STATE Update \begin{align*}
	        \mathscr{B}^{k}=\left\{f \in \cF: L_h^{1:k}(D^{1:k}_h, f, f)-\inf_{g \in \cG}L_h^{1:k}(D^{1:k}_h,f, g)\le \beta, \forall h \in [H] \right\}.
	    \end{align*}\label{line:adabatch l2confidenceset}
            \STATE $length = 1$.
            \ELSE \STATE $\mathscr{B}^{k}=\mathscr{B}^{k-1}$.
            \STATE $length = length + 1$.
            \ENDIF\ENDIF\label{alg:adabatch l2endif}
	\ENDFOR
 \end{algorithmic}
\end{algorithm}

\subsection{Proof of Theorem \ref{thm:batchl1}}\label{appendix:batch l1}

\begin{proof}
The proof for the batch learning problem under $\ell_1$-EC class is straightforward. 
    We fix a $h \in [H]$ in the proof. Since we change our policy at time $k_j = j\cdot \lfloor K/B \rfloor + 1$ for $j\ge 0$, we can get 
    \begin{align*}
        \sum_{i=1}^{k_j-1}d_{\mathrm{TV}}(\PP^{\pi^i}_{f^{k_j}}, \PP^{\pi^i}_{g^{k-1}}) \le c\sqrt{\beta k_j}.
    \end{align*}
    Then by Eq.\eqref{eq:TV6betak} in \S \ref{appendix:proof of l1}, for all $j \ge 0$, we have 
    \begin{align}\label{eq:batchl16c}
        \sum_{i=1}^{k_j-1}d_{\mathrm{TV}}(\PP^{\pi^i}_{f^{k_j}}, \PP^{\pi^i}_{f^*}) \le 6c\sqrt{\beta k_j}.
    \end{align}
    By our batch learning algorithm, for $k_{j}\le k< k_{j+1}$, $\pi^k$ is the same policy. Thus, we can transform Eq.\eqref{eq:batchl16c} to
    \begin{align}
        \sum_{i=0}^{j-1} d_{\mathrm{TV}}(\PP^{\pi^{k_{i}}}_{f^{k_j}}, \PP^{\pi^{k_{i}}}_{f^*}) \le 6c\sqrt{\beta k_j}\cdot \frac{B}{K} \le 12c\sqrt{\frac{\beta B}{K} j}.
    \end{align}
    Then by the $\ell_1$-type eluder condition, we can get 
    \begin{align}
        \sum_{i=0}^{j-1} d_{\mathrm{TV}}(\PP^{\pi^{k_{i}}}_{f^{k_i}}, \PP^{\pi^{k_{i}}}_{f^*})  \le \widetilde{\cO}\left(\mathrm{poly}\log (K)\cdot\left(\sqrt{\frac{d\beta B}{K}j}+\sqrt{d}\right)\right).\label{ineq:final TV bound}
    \end{align}
    Then we have 
    \begin{align*}
        R(K) &= \sum_{i=1}^K (V_{1,f^*}^{\pi^*}(s_1) - V_{1,f^*}^{\pi^i}(s_1)) \le \sum_{i=1}^K (V_{1,f^i}^{\pi^i}(s_1)-V_{1,f^*}^{\pi^i}(s_1))\\
        &\le \sum_{i=1}^K H\cdot d_{\mathrm{TV}}(\PP^{\pi^i}_{f^i}, \PP^{\pi^i}_{f^*}) \le H\cdot \left\lfloor\frac{K}{B}\right\rfloor\cdot  \sum_{i=0}^{\lceil K/B\rceil}d_{\mathrm{TV}}(\PP^{\pi^{k_{i}}}_{f^{k_i}}, \PP^{\pi^{k_{i}}}_{f^*}),
    \end{align*}
where the first inequality is because $f^* \in \cB^k$ for all $k \in [K]$ and $f^i$ is the optimal policy within the confidence set $\cB^i.$ The second inequality holds because of the definition of TV distance, and the last inequality holds because the $f^{j}$ and  $\pi^j$ are the same within the same batch $j \in [k_i,k_{i+1}).$

Then by the Eq.\eqref{ineq:final TV bound}, we can get 
\begin{align*}
    R(K) &\le  H\cdot \left\lfloor\frac{K}{B}\right\rfloor\cdot  \sum_{i=0}^{\lceil K/B\rceil}d_{\mathrm{TV}}(\PP^{\pi^{k_{i}}}_{f^{k_i}}, \PP^{\pi^{k_{i}}}_{f^*})\\
        &\le \widetilde{\cO}\left(H\cdot \left\lfloor\frac{K}{B}\right\rfloor\cdot\mathrm{poly}\log (K)\left(\sqrt{\frac{d\beta B}{K}B}+\sqrt{d}\right)\right)\\
        &= \widetilde{\cO}\left(H\cdot \left\lfloor\frac{K}{B}\right\rfloor \cdot\mathrm{poly}\log (K) \left(B\sqrt{\frac{d\beta}{K}}+\sqrt{d}\right)\right)\\
        &\le \widetilde{\cO}\left(H\cdot\mathrm{poly}\log (K)\left(\sqrt{d}\cdot \frac{K}{B}+\sqrt{d\beta K}\right)\right).
\end{align*}
We complete the proof of Theorem \ref{thm:batchl1}.
\end{proof}

\section{Additional Examples for \texorpdfstring{$\ell_2$}{}-type EC Class}\label{appendix:additional examples}
In this section, we provide some additional examples in our $\ell_2-$type EC class. 
\subsection{Generalized Linear Bellman Complete}
Generalized Linear Bellman Complete is introduced in \cite{du2021bilinear, chen2022abc}, which consists of a link function $\sigma:\RR\to \RR^{+}$ with $\sigma'(x) \in [L,U]$ for $0<L<U$, and a hypothesis class $\cF = \{\cF_h = \sigma(\theta_h^T\phi(s,a)):\theta_h \in \cH_h\}$ with $\Vert\theta_h\Vert_2\le 1$. Also, for any $f \in \cF$, the Bellman complete condition holds:
\begin{align*}
    r(s,a) + \EE_{s'} \Big[\max_{a' \in \cA} \sigma(\theta_{h+1,f}^T\phi(s',a')) \Big]\in \cF_h.
\end{align*}
Hence we know there is a mapping $\cT:\cH\to \cH$ such that 
\begin{align*}
    r(s,a) + \EE_{s'}\Big [\max_{a' \in \cA} \sigma(\theta_{h+1,f}^T\phi(s',a'))\Big ] = \sigma(\cT(\theta_{h+1,f})^T\phi(s,a)).
\end{align*}
If we let 
\begin{align*}
    \ell_{h,f'}(s_{h+1},\{s_h,a_h\},f,g) = \sigma(\theta_{h,g}^T\phi(s_h,a_h))-r_h(s_h,a_h)-\max_{a' \in \cA}\theta_{h+1,f}^{T}\phi(s_{h+1},a'),
\end{align*}
then the expectation can be written as 
\begin{align*}
    \EE_{s_{h+1}}\Big[\ell_{h,f'}(s_{h+1},\{s_h,a_h\},f,g)\Big] &= \sigma(\theta_{h,g}^T\phi(s_h,a_h))-\sigma(\cT(\theta_{h+1,f})^T\phi(s_h,a_h)).
\end{align*}
Thus we have 
\begin{align*}
    \EE_{s_{h+1}}\Big[\ell_{h,f'}(s_{h+1},\{s_h,a_h\},f,g)\Big]^2 \in [L^2(\theta_{h,g}^T-\theta_{h+1,f}^T)\phi(s_h,a_h))^2, U^2(\theta_{h,g}^T-\theta_{h+1,f}^T)\phi(s_h,a_h))^2].
\end{align*}
Now 
following the similar analyses in Section \ref{sec:proof of linear mixture} with $X_h(f^k) = \theta_{h,f}^T-\theta_{h+1,f}^T, W_{h,f^k}(s_h,a_h) = \phi(s_h,a_h)$, we can show that it is a DLF and satisfies the $\ell_2$-type condition and dominance property. The constant $L$ and $U$ will only influence the final $\ell_2$-type condition by a constant and could be ignored in the $\cO(\cdot)$ notation.

\subsection{Linear \texorpdfstring{$Q^*/V^*$}{}}
The Linear $Q^*/V^*$ model is proposed in \cite{du2021bilinear}. In this model, the optimal $Q$-value and $V$-value functions have a linear structure: There are two known features $\phi(s,a)$ and $\psi(s')$  with unknown parameters $\omega^*, \theta^*$ such that 
\begin{align*}
    Q^*(s,a) = \langle\phi(s,a), \omega_h^*\rangle, \ V^*(s) = \langle \psi(s), \theta_h^*\rangle.
\end{align*}

Denote our hypothesis class as $\cF = \cF_1\times \cdots \times \cF_H$, where $\cF_h$ is defined as   
\begin{align*}
    \{f = (\omega, \theta): \max_{a \in \cA}\omega^T\phi(s,a) = \theta^T\psi(s), \ \forall \ s \in \cS\}.
\end{align*}
Then we denote the loss function as 
\begin{align*}
    &\ell_{h,f'}(s_{h+1},\{s_h,a_h\},f,g) \\&\quad= Q_{h,g}(s_h,a_h)-r_h-V_{h+1,f}(s_{h+1}) = \omega_{h,g}^T\phi(s_h,a_h)-r_h-\theta_{h+1,f}^T\psi(s_{h+1}),
\end{align*}
and we can calculate the expectation by 
\begin{align*}
    \EE_{s_{h+1}}\Big[\ell_{h,f'}(s_{h+1},\{s_h,a_h\},f,g)\Big]=
    (\omega_{h,g}-\omega^*, \theta_{h+1,f}-\theta^*)^T \EE_{s_{h+1}}\Big[\phi(s_h,a_h),\psi(s_{h+1})\Big].
\end{align*}
Note that the expectation has a bilinear structure that is similar to the linear mixture MDP \eqref{eq:linear mixture expectation loss}, then if we choose $X_h(f^k) = (\omega_{h,f^k}-\omega^*, \theta_{h+1,f^k}-\theta^*)$ and $W_{h,f^k}(s_h,a_h) = \EE_{s_{h+1}}[\phi(s_h,a_h),\psi(s_{h+1})]$, we can apply an argument similar to Section \ref{sec:proof of linear mixture} to prove the $\ell_2$-type condition, dominance property and decomposable property. 
\subsection{Linear Quadratic Regulator}
 In the Linear Quadratic Regulator (LQR) model  \citep{bradtke1992reinforcement}, we consider $d_s$-dimensional state space $\cS\subseteq \RR^{d_s}$ and $d_a$-dimensional action space $\cA\subseteq \RR^{d_a}$, then a LQR model consists of unknown matrix $A \in \RR^{d_s\times d_s}, B \in \RR^{d_s\times d_a}$ and $P \in \RR^{d_s\times d_s}$ such that 
\begin{align*}
    s_{h+1} = As_h + Ba_h + \varepsilon_h, \ \ r_h(s_h,a_h) = s_h^TQs_h+ a_h^Ta_h + \varepsilon'_h,
\end{align*}
where $\varepsilon_h, \varepsilon'_h$ are zero-centered random noises with $\EE[s_hs_h^T] = \Sigma$ and $\EE[(\varepsilon'_h)^2]=\sigma^2$.

The LQR model has been extensively analyzed \citep{du2021bilinear,chen2022abc}. By Lemma A.3 in \cite{du2021bilinear}, the hypothesis class are $\cF = \{(C_h, \Lambda_h, O_h):C_h \in \RR^{d_a\times d_s}, \Lambda_h \in \RR^{d_s\times d_s}, O_h \in \RR\},$ and \begin{align*}
    \pi_f(s_h) = C_{h,f}(s_h), \ \ V_{h,f}(s_h) = s_h^T\Lambda_{h,f}s_h + O_{h,f}.
\end{align*}
Let 
\begin{align*}
    \ell_{h,f'}(s_{h+1},\{s_h,a_h\},f,g) = Q_{h,g}(s_h,a_h)-r_h-V_{h+1,f}(s_{h+1}),
\end{align*}
then the expectation
\begin{align*}
    &\EE_{s_{h+1}}\Big[\ell_{h,f'}(s_{h+1},\{s_h,a_h\},f,g)\Big]\\ &\quad=\Big\langle \mbox{vec}(\Lambda_{h,f}-Q-C_{h,f}^TC_{h,f}-(A+BC_{h,f})^T\Lambda_{h+1,f}(A+BC_{h,f})),\\ &\qquad \quad O_{h,f}-O_{h+1,f}-\tr(\Lambda_{h+1,f}\Sigma))^T\mbox{vec}(s_hs_h^T,1)\Big\rangle
\end{align*}
has a bilinear structure that is similar to the linear mixture MDP \eqref{eq:linear mixture expectation loss}, then if we choose 
\begin{small}
\begin{align*}
    X_h(f^k) = \mbox{vec}(\Lambda_{h,f}-Q-C_{h,f}^TC_{h,f}-(A+BC_{h,f})^T\Lambda_{h+1,f}(A+BC_{h,f})), O_{h,f}-O_{h+1,f}-\tr(\Lambda_{h+1,f}\Sigma))
\end{align*}
\end{small}and 
\begin{align*}
    W_{h,f^k}(s_h,a_h) = \mbox{vec}(s_hs_h^T,1)
\end{align*} in the analyses of Section \ref{sec:proof of linear mixture}, we can apply an argument similar to Section \ref{sec:proof of linear mixture} to prove the $\ell_2$-type condition, dominance property and decomposable property.

\section{Proof of Lemmas}
\subsection{Proof of Lemma~\ref{lemma:optimal -beta}}
\begin{proof}
First note that we choose $\zeta_h = s_{h+1}$ , $ \eta_h = \{s_h,a_h\}$ for MDP and $\eta_h = \{s_h,a_h,b_h\}$ for zero-sum Markov Games.
    Define the auxillary variable $X_{i,f'}(h,f,g) = \Vert\ell_{h,f'} (s_{h+1}^i, \eta_h^i, f, g)\Vert_2^2 - \Vert\ell_{h,f'}(s_{h+1}^i, \eta_h^i, f, \mathcal{T}(f))\Vert_2^2,$ then we know $|X_{i,f'}(h,f,g)|\le 2R^2$ for all $1\le i\le k$, where $R = \sup \Vert \ell_{h,f'}(\zeta_h,\eta_h,f,g)\Vert_2$. Now 
     we can have
     \begin{align}
         &\mathbb{E}_{  s_{h+1}}\Big[X_{i,f'}(h,f,g)  \Big] \nonumber\\&\quad=\mathbb{E}_{  s_{h+1} } \Big[\langle \ell_{h,f'}(s_{h+1}, \eta_h^i, f, g)-\ell_{h,f'}(s_{h+1}, \eta_h^i, f, \mathcal{T}(f)),\nonumber\\&\qquad \ell_{h,f'}(s_{h+1}, \eta_h^i, f, g)+\ell_{h,f^i}(\zeta_h, \eta_h^i, f, \mathcal{T}(f))\rangle\Big]\nonumber\\
         &\quad = \mathbb{E}_{ s_{h+1} }\Big[\langle \mathbb{E}_{s_{h+1}}[\ell_{h,f'}(s_{h+1}, \eta_h^i, f,g) ], \ell_{h,f'}(s_{h+1}, \eta_h^i, f, g)\rangle \Big]\label{eq:apply dec}\\
         & \quad= \Big\Vert\mathbb{E}_{s_{h+1}}\Big[\ell_{h,f'}(s_{h+1}, \eta_h^i, f, g)\Big]\Big\Vert_2^2. \label{eq:firstorderterm}
     \end{align}
    The Eq.\eqref{eq:apply dec} holds from the decomposable property of $\ell$. Then we have 
     \begin{align}
         &\mathbb{E}_{   s_{h+1}}\Big[( X_{i,f'}(h,f,g))^2  \Big] \nonumber\\&\quad  
        = \mathbb{E}_{   s_{h+1}}\Big[\Vert\ell_{h,f'}(s_{h+1}, \eta_h^i, f, g)-\ell_{h,f^i}(\zeta_h, \eta_h^i, f, \mathcal{T}(f))\Vert_2^2\nonumber\\&\qquad\cdot \Vert\ell_{h,f'}(s_{h+1}, \eta_h^i, f, g)+\ell_{h,f'}(s_{h+1}, \eta_h^i, f, \mathcal{T}(f))\Vert_2^2 \Big]\nonumber\\&\quad\le 
         4R^2\mathbb{E}_{   s_{h+1}}\Big[\Vert\ell_{h,f'}(s_{h+1}, \eta_h^i, f, g)-\ell_{h,f'}(s_{h+1}, \eta_h^i, f, \mathcal{T}(f))\Vert_2^2\Big]\label{ineq:only inequality}\\
         &\quad= 4R^2 \mathbb{E}_{   s_{h+1}}\Big[\Vert\mathbb{E}_{ s_{h+1}}[\ell_{h,f'}(s_{h+1}, \eta_h^i, f, g) ]\Vert_2^2\Big]\label{eq:decomposable apply}\\& \quad= 
         4R^2\Big\Vert\mathbb{E}_{s_{h+1}}\Big[\ell_{h,f^i}(\zeta_h, \eta_h^i, f, g)\Big]\Big\Vert_2^2\nonumber\\
         &\quad= 4R^2\mathbb{E}_{   s_{h+1}}\Big[ X_{i,f'}(h,f,g)  \Big],\label{eq:secondorder}
     \end{align}
     where the inequality \eqref{ineq:only inequality} is because $\Vert \ell_{h,f'}(s_{h+1},\eta_h^i,f,g)\Vert_2\le R$ for any $h$, $(f',f,g)$ and $s_{h+1,\eta_h^i}.$ The Eq.\eqref{eq:decomposable apply} holds from the decomposable property, and the Eq.\eqref{eq:secondorder} holds from the Eq.\eqref{eq:firstorderterm}.
     Thus by Freedman's inequality \citep{agarwal2014taming, jin2021bellman,chen2022abc}, 
     with probability at least $1-\delta$, 
     \begin{align*}
         &\left|\sum_{i=a}^b  X_{i,f'}(h,f,g) -  \sum_{i=a}^b \mathbb{E}_{\zeta_h}[ X_{i,f'}(h,f,g) ]\right|\\ 
         &\quad\le \cO\left(R\sqrt{\log(1/\delta)\sum_{i=a}^b \mathbb{E}_{\zeta_h}\Big[ X_{i,f'}(h,f,g) \Big]}+2R^2\log(1/\delta)\right).
     \end{align*}
     Now we consider a $\rho$-cover $L_{h,\rho} = (\tilde{\cF}_\rho, \tilde{\cF}_\rho, \tilde{\cG}_\rho)$ for $(\cF, \cG)$: For any $f,f' \in \cF, g \in \cG$, there exists a pair of function $(\tilde{f'}, \tilde{f}, \tilde{g}) \in (\tilde{\cF}_\rho,\tilde{\cF}_\rho, \tilde{\cG}_\rho)$ such that 
     $\Vert\ell_{h,\tilde{f}'}(\cdot, \tilde{f}, \tilde{g})-\ell_{h,f'}(\cdot, f,g)\Vert_\infty\le \rho$.
     By taking a union bound over $L_{h,\rho}$, $a \in [K], b \in [K], h \in [H]$, we can have 
     \begin{align*}
         &\left|\sum_{i=a}^b  X_{i,\tilde{f'}}(h,\tilde{f},\tilde{g}) -  \sum_{i=a}^b \mathbb{E}_{\zeta_h}\Big[ X_{i,\tilde{f'}}(h,\tilde{f},\tilde{g}) \Big]\right|\\ 
         &\quad\le \cO\left(R\sqrt{\iota\sum_{i=a}^b \mathbb{E}_{\zeta_h}\Big[ X_{i,\tilde{f'}}(h,\tilde{f},\tilde{g}) \Big]}+2R^2\iota\right),
     \end{align*}
    where  $\iota = \log(HK^2\cN_\cL(1/K)/\delta)$ and $\cN_\cL(1/K) = \max_h |L_{h,\rho}|$ is the maximum $\rho$-covering number of $(\cF,\cF, \cG)$ for loss function $\ell_{h, f'}(\cdot,\cdot, f, g)$ for $h \in [H]$.
        
    Since 
    $\EE_{\zeta_h}[ X_{i,f'}(h,\tilde{f}, \tilde{g}) ] \ge 0$ and \begin{align*}-\sum_{i=a}^b  X_{i,\tilde{f'}}(h,\tilde{f},\tilde{g})&\le \cO(R^2\iota), \ \ \forall \ (\tilde{f}, \tilde{g}) \in L_{h,\rho},\\
        -\sum_{i=a}^b  X_{i,f'}(h,f,g)&\le \cO(R^2\iota+R) \le \beta, \ \ \forall \ f \in \cF, g \in \cG.
        \end{align*}
    where $\beta = c(R^2\iota + R)$ for some large enough constant $c$.
\end{proof}
\subsection{Proof of Lemma~\ref{lemma:Cbeta to C+1beta}}
\begin{proof}
    The proof is similar to Lemma~\ref{lemma:optimal -beta}. Apply the same covering argument and concentration inequality, for any $(\tilde{f'}, \tilde{f},\tilde{g})\in L_{h,\rho}$, we have 
    \begin{align}
         &\left|\sum_{i=1}^{k-1} X_{i,\tilde{f'}}(h,\tilde{f},\tilde{g}) -  \sum_{i=1}^{k-1} \mathbb{E}_{\zeta_h}\Big[X_{i,\tilde{f'}}(h,\tilde{f},\tilde{g}) \Big]\right|\nonumber\\ 
         &\quad\le \cO\left(R\sqrt{\iota\sum_{i=1}^{k-1} \mathbb{E}_{\zeta_h}\Big[X_{i,\tilde{f'}}(h,\tilde{f},\tilde{g}) \Big]}+2R^2\iota\right),\label{eq:freedmanresult}
     \end{align}
     where  $\iota = \log(HK^2\cN_\cL(1.T)/\delta)$ and $\cN_\cL(\rho) = \max_h |L_{h,\rho}|$ is the maximum $\rho$-covering number of $(\cF,\cF, \cG)$ for loss function $\ell_{h, f'}(\cdot,\cdot, f, g)$ for $h \in [H]$.
     Now note that 
     \begin{align*}
         &\sum_{i=1}^{k-1} X_{i,f^i}(h,f^k,f^k) \\&\quad= \sum_{i=1}^{k-1}  (\Vert\ell_{h,f^i} (\zeta_h^i, \eta_h^i, f^k, f^k)\Vert_2^2 - \Vert\ell_{h,f^i}(\zeta_h^i, \eta_h^i, f^k, \mathcal{T}(f^k))\Vert_2^2)\\
         &\quad\le C\beta.
     \end{align*}
     Then there is a pair $(\tilde{f'}, \tilde{f},\tilde{g})\in L_{h,1/K}$ such that 
     \begin{align}
         \Bigg|\sum_{i=1}^{k-1} X_{i, \tilde{f'}}(h,\tilde{f},\tilde{g})-\sum_{i=1}^{k-1} X_{i,f^i}(h,f^k,f^k) \Bigg|\le \cO(R)\nonumber
     \end{align}
     and 
     \begin{align}
         \sum_{i=1}^{k-1} X_{i,\tilde{f'}}(h,\tilde{f},\tilde{g})\le C\beta + \cO(R).\nonumber
     \end{align}
     Combining with Eq.~\eqref{eq:freedmanresult}, when $\beta= c(R^2\iota+R)$ for sufficiently large constant $c$, if $C \le 100$ is a small constant, we get 
     \begin{align}
         \sum_{i=1}^{k-1}\EE_{\zeta_h}\Big[X_{i,\tilde{f'}}(h,\tilde{f},\tilde{g}) \Big]\le \left(C+\frac{1}{2}\right)\beta.\nonumber
     \end{align}
     Since $(\Tilde{f'}, \Tilde{f}, \Tilde{g})$ is the $\rho-$approximation of $(f^i, f^k, f^k)$,
     \begin{align}\label{eq:finalC+1beta}
         \sum_{i=1}^{k-1}\Big\Vert\EE_{\zeta_h}\Big[\ell_{h,f^i}(\zeta_h, \eta_h^i, f^k,f^k)\Big]\Big\Vert_2^2=\sum_{i=1}^{k-1}\EE_{\zeta_h}\Big[X_{i,f^i}(h,f^k,f^k) \Big]\le (C+1)\beta.
     \end{align}
     The first equality of Eq.~\eqref{eq:finalC+1beta} is derived from Eq.~\eqref{eq:firstorderterm}.

     If $C\ge 100$, similarly we can get 
     \begin{align}
         \sum_{i=1}^{k-1}\EE_{\zeta_h}\Big[X_{i,\tilde{f'}}(h,\tilde{f},\tilde{g}) \Big]\le \left(2C-1\right)\beta\nonumber
     \end{align}
     and
     \begin{align}\label{eq:final2Cbeta}
         \sum_{i=1}^{k-1}\Big\Vert\EE_{\zeta_h}\Big[\ell_{h,f^i}(\zeta_h, \eta_h^i, f^k,f^k)\Big]\Big\Vert_2^2=\sum_{i=1}^{k-1}\EE_{\zeta_h}\Big[X_{i,f^i}(h,f^k,f^k) \Big]\le (2C)\beta.
     \end{align}

     Similar to \citep{jin2021bellman}, to prove the second inequality Eq.\eqref{eq:C+1beta second inequality} we can add the $\eta_h$ to expectation by replacing the $\EE_{s_{h+1}}[X_{i,f'}(h,f,g)]$ to $\EE_{s_h,a_h\sim \pi^i}\EE_{s_{h+1}}[X_{i,f'}(h,f,g)]$.
\end{proof}
\subsection{Proof of Lemma~\ref{lemma:Cbeta to C-1beta}}
\begin{proof}
    The proof is similar to Lemma~\ref{lemma:optimal -beta}. Now apply the same argument, for any $(\tilde{f'}, \tilde{f},\tilde{g})\in L_{h,\rho}$, with probability at least $1-\delta$, we have 
    \begin{align}
         &\left|\sum_{i=1}^{k-1} X_{i,\tilde{f'}}(h,\tilde{f},\tilde{g}) -  \sum_{i=1}^{k-1} \mathbb{E}_{\zeta_h}[X_{i,\tilde{f'}}(h,\tilde{f},\tilde{g}) ]\right|\nonumber\\ 
         &\quad\le \cO\left(R\sqrt{\iota\sum_{i=1}^{k-1} \mathbb{E}_{\zeta_h}\Big[X_{i,\tilde{f'}}(h,\tilde{f},\tilde{g}) \Big]}+2R^2\iota\right),\label{eq:freedmanresult2}
     \end{align}
     where  $\iota = \log(HK^2\cN_\cL(1/K)/\delta)$ and $\cN_\cL(\rho) = \max_h |L_{h,\rho}|$ is the maximum $\rho$-covering number of $(\cF,\cF, \cG)$ for loss function $\ell_{h, f'}(\cdot,\cdot, f, g)$ for $h \in [H]$.
     Now note that 
     \begin{align*}
         &\sum_{i=1}^{k-1} X_{i,f'}(h,f^k,f^k) \\&\quad= \sum_{i=1}^{k-1}  (\Vert\ell_{h,f^i} (\zeta_h^i, \eta_h^i, f^k, f^k)\Vert_2^2 - \Vert\ell_{h,f^i}(\zeta_h^i, \eta_h^i, f^k, \mathcal{T}(f^k))\Vert_2^2)\\
         &\quad\ge C\beta.
     \end{align*}
     Then there is a pair $(\tilde{f'}, \tilde{f},\tilde{g})\in L_{h,1/K}$ such that 
     \begin{align}
         \Bigg|\sum_{i=1}^{k-1} X_{i,\tilde{f'}}(h,\tilde{f},\tilde{g})-\sum_{i=1}^{k-1} X_{i,f'}(h,f^k,f^k) \Bigg|\le \cO(R),\nonumber
     \end{align}
     and 
     \begin{align}
         \sum_{i=1}^{k-1} X_{i,\tilde{f'}}(h,\tilde{f},\tilde{g})\ge C\beta - \cO(R).\nonumber
     \end{align}
     Combining with Eq.~\eqref{eq:freedmanresult2}, when $\beta= c(R^2\iota+R)$ for sufficiently large constant $c$, if $C\le 100$ we get 
     \begin{align}
         \sum_{i=1}^{k-1}\EE_{\zeta_h}\Big[X_{i,\tilde{f'}}(h,\tilde{f},\tilde{g})\Big ]\ge \left(C-\frac{1}{2}\right)\beta\nonumber.
     \end{align}
     Since $(\Tilde{f'}, \Tilde{f}, \Tilde{g})$ is a $\rho-$approximation of $(f^i, f^k, f^k)$,
     \begin{align}\label{eq:finalC-1beta}
         \sum_{i=1}^{k-1}\Big\Vert\EE_{\zeta_h}\Big[\ell_{h,f^i}(\zeta_h, \eta_h^i, f^k,f^k)\Big]\Big\Vert_2^2&=\sum_{i=1}^{k-1}\EE_{\zeta_h}\Big[X_{i,f^i}(h,f^k,f^k) \Big]\nonumber\\&\ge \left(C-\frac{1}{2}\right)\beta - \cO(R)\nonumber\\&\ge (C-1)\beta.
     \end{align}
        the first equality of Eq.~\eqref{eq:finalC-1beta} is derived from Eq.~\eqref{eq:firstorderterm}.

        Similarly, if $C\ge 100$, we can get \begin{align}
         \sum_{i=1}^{k-1}\EE_{\zeta_h}\Big[X_{i,\tilde{f'}}(h,\tilde{f},\tilde{g})\Big ]\ge \left(C/2+1\right)\beta\nonumber
     \end{align}
     and
     \begin{align}\label{eq:finalC/2beta}
         \sum_{i=1}^{k-1}\Big\Vert\EE_{\zeta_h}\Big[\ell_{h,f^i}(\zeta_h, \eta_h^i, f^k,f^k)\Big]\Big\Vert_2^2&=\sum_{i=1}^{k-1}\EE_{\zeta_h}\Big[X_{i,f^i}(h,f^k,f^k) \Big]\nonumber\\&\ge \left(C/2+1\right)\beta - \cO(R)\nonumber\\&\ge (C/2)\beta.
     \end{align}

        Similar to Lemma~\ref{lemma:Cbeta to C+1beta}, to prove the second inequality Eq.\eqref{eq:C-1beta second inequality}, we can add the $\eta_h$ to expectation by replacing the $\EE_{s_{h+1}}[X_{i,f'}(h,f,g)]$ to $\EE_{\eta_h\sim \pi^i}\EE_{s_{h+1}}[X_{i,f'}(h,f,g)]$.
\end{proof}
\subsection{Proof of Lemma~\ref{lemma:adaptive batch}}
Assume that at episode $k$ we have $L_h^{1:k}(D^{1:k}_h, f^k, f^k)-\inf_{g \in \cG}L_h^{1:k}(D^{1:k}_h,f^k, g)\ge 24\beta$, and it belongs to a batch with length $2^i$. First by Lemma~\ref{lemma:optimal -beta}, $$L_h^{1:k}(D^{1:k}_h, f^k, f^k)-L_h^{1:k}(D^{1:k}_h,f^k, \cT(f^k))\ge23\beta.$$
Suppose $i=0$ and this batch starts at $k'$, thus at episode $k'$ we have
\begin{align*}
L_h^{1:k'}(D^{1:k'}_h, f^{k}, f^{k})-L_h^{1:k'}(D^{1:k}_h,f^{k}, \cT(f^k))&\le L_h^{1:k'}(D^{1:k'}_h, f^{k}, f^{k})-\inf_{g \in \cG}L_h^{1:k'}(D^{1:k}_h,f^{k}, g)\\&\le \beta,
\end{align*}
where the second inequality is the property of optimistic confidence set. Then by these two inequalities, 
\begin{align*}
    L_h^{k'+1:k}(D^{k'+1:k}_h, f^k, f^k)-L_h^{k'+1:k}(D^{k'+1:k}_h,f^k, \cT(f^k))\ge 22\beta.
\end{align*}
This is impossible because $i = 0$ and thus $k = k' + 2^i = k' + 1$, and $\beta\ge 2R^2\ge L_h^{k'+1:k}(D^{k'+1:k}_h, f^k, f^k)-L_h^{k'+1:k}(D^{k'+1:k}_h,f^k, \cT(f^k))$, where $R$ is the upper bound of the norm of loss function $\ell$.

Now suppose $i \ge 1$, then 
\begin{align*}
L_h^{1:k'}(D^{1:k'}_h, f^{k}, f^{k})-L_h^{1:k'}(D^{1:k}_h,f^{k}, \cT(f^k))&\le L_h^{1:k'}(D^{1:k'}_h, f^{k}, f^{k})-\inf_{g \in \cG}L_h^{1:k'}(D^{1:k}_h,f^{k}, g)\\&\le 5\beta,
\end{align*}
where the second inequality is derived by the updating rule with $i\neq 0$, and the fact that $f^k = f^{k'}.$
Then we can also get 
\begin{align*}
    L_h^{k'+1:k}(D^{k'+1:k}_h, f^k, f^k)-L_h^{k'+1:k}(D^{k'+1:k}_h,f^k, \cT(f^k))\ge 24\beta - 5\beta = 19\beta.
\end{align*}
Note that we only consider the MDP problem here, then $\pi^i, f^i$ are the same for $i \in [k'+1,k]$.
By combining this fact and  Lemma \ref{lemma:Cbeta to C-1beta}, we can get 
\begin{align}
    (k-k'-1) \cdot \EE_{\eta_h \sim \pi^k}\Big\Vert\EE_{\zeta_h}\Big[\ell_{h,f^k}(\zeta_h, \eta_h, f^k, f^k)\Big]\Big\Vert_2^2 &=  \sum_{i=k'+1}^{k} \EE_{\eta_h \sim \pi^i}\Big\Vert\EE_{\zeta_h}\Big[\ell_{h,f^i}(\zeta_h, \eta_h, f^k, f^k)\Big]\Big\Vert_2^2 \nonumber\\&\ge 19\beta - \beta \nonumber\\&= 18\beta.\label{eq:greater12beta}
\end{align}
However, since $k \in [ k' + 2^i, k'+2^{i+1})$ with $i \ge 1$, we know for $t \in [k'+2^{i-1}, k'+2^i)$, $f^t = f^k$ because at episode $k'+2^{i-1}$, the agent does not change the policy by the definition of $k'$. Thus, denote $k_1 = k' + 2^{i-1}+1$, $k_2 = k'+2^i$, we have 
\begin{align*}
    L_h^{k_1:k_2}(D_h^{k_1:k_2}, f^{k_2}, f^{k_2}) - L_h^{k_1:k_2}(D_h^{k_1:k_2}, f^{k_2}, \cT(f^{k_2})) \le 5\beta,
\end{align*}
and 
\begin{align}
    (k_2-k_1)\cdot \EE_{\eta_h \sim \pi^k}\Big\Vert\EE_{\zeta_h}\Big[\ell_{h,f^k}(\zeta_h, \eta_h, f^{k}, f^{k})\Big] \Big\Vert_2^2& = \sum_{i=k_1}^{k_2} \EE_{\eta_h \sim \pi^i}\Big\Vert\EE_{\zeta_h}\Big[\ell_{h,f^i}(\zeta_h, \eta_h, f^{k_2}, f^{k_2})\Big]\Big\Vert_2^2 \nonumber\\&\le 6\beta. \label{eq:lessthan6beta}
\end{align}
Note that $3\cdot (k_2-k_1) = 3\cdot (2^{i-1}-1) > 2^i \ge k-k'-1$, we know Eq.\eqref{eq:greater12beta} and Eq.\eqref{eq:lessthan6beta} cannot both hold. Hence we have done the proof by contradiction.

\subsection{Proof of Lemma~\ref{lemma:BEdim l2}}

First, we introduce the definition of $D_\Delta$-type Bellman eluder dimension.
\begin{definition}[Bellman eluder dimension]
Given a function class $\cF$, the $D_\Delta$ Bellman eluder dimension $d(\cF, D_\Delta, \varepsilon)$ is the length $n$ of longest sequence $((s_h^1, a_h^1), \cdots, (s_h^n, a_h^n))$ such that for some $\varepsilon'\ge \varepsilon$ and any $j \in [n]$, there exists a $f^j \in \cF$ and $\sqrt{\sum_{i=1}^{j-1} \cE_h(f^j, s_h^i, a_h^i)^2}\le \varepsilon$ and $\cE_h(f^j, s_h^j, a_h^j)>\varepsilon.$ The term $\cE_h(f,s,a)$ is the Bellman error $\cE_h(f,s,a)= (f_h - \cT(f_{h+1}))(s,a)$.
\end{definition}
\begin{proof}
Now we begin to prove the Lemma~\ref{lemma:BEdim l2}. 
First, we restate the Proposition 43 in \cite{jin2021bellman} with $\Pi = D_\Delta$.
\begin{proposition}[Proposition 43 in \cite{jin2021bellman} with $\Pi = D_\Delta$]
    Given a function class $\Phi$ defined on $\cX$,  suppose given sequence $\{\phi_i\}_{1\le i\le K}\subset \Phi$  and sequences $\{(s_h^i,a_h^i)\}_{i \le [K]}$ such that for all $k \in [K]$, $\sum_{i=1}^k (\EE_{\mu_i}[\phi_k])^2\le \beta$, then for all $k \in [K]$,
    \begin{align*}
        \sum_{i=1}^k \II\{|\EE_{\mu_i}[\phi_i]|>\varepsilon\}\le \left(\frac{\beta}{\varepsilon^2}+1\right)d_{\mathrm{BE}}(\Phi, D_\Delta, \varepsilon). 
    \end{align*}
\end{proposition}

Now,  we first fixed a $h \in [H],$ then choosing $\Pi = D_\Delta$ , $\phi_i = \cE_h(f_i, s,a)$ and $\mu_i = \II\{(s,a) = (s_h^i, a_h^i)\}$ in proposition 43, based on since $\sum_{i=1}^{k-1}\cE_h(f^k, s_h^i, a_h^i)\le \beta$ for all $h, k$, we have 
\begin{align*}
    \sum_{i=1}^k \II\{\cE_h(f^i, s_h^i, a_h^i)^2>\varepsilon^2\}\le \left(\frac{\beta}{\varepsilon^2}+1\right)d_{\mathrm{BE}}(\cF, D_\Delta, \varepsilon).
\end{align*}
Then by replacing $\varepsilon^2$ to $\varepsilon$,
\begin{align*}
    \sum_{i=1}^k \II\{\cE_h(f^i, s_h^i, a_h^i)^2>\varepsilon\}\le \left(\frac{\beta}{\varepsilon}+1\right)d_{\mathrm{BE}}(\cF, D_\Delta, \sqrt{\varepsilon}).
\end{align*}
Now sort the sequence $\{\cE_h(f^1, s_h^1, a_h^1)^2, \cE_h(f^2, s_h^2, a_h^2)^2,\cdots, \cE_h(f^k, s_h^k, a_h^k)^2\}$ in a decreasing order and denote them by $\{e_1,\cdots,e_k\}$, for any $\omega$ we can have 
\begin{align*}
    \sum_{i=1}^k e_i = \sum_{i=1}^k e_i\II\{e_i\le \omega\}+\sum_{i=1}^k e_i \II\{e_i> \omega\} \le k\omega + \sum_{i=1}^k e_i \II\{e_i> \omega\}.
\end{align*}
Assume $e_t>\omega$ and there exists a parameter $\alpha \in (\omega, e_t)$, then 
\begin{align*}
    t\le \sum_{i=1}^k \II\{e_t>\alpha\}\le \left(\frac{\beta}{\alpha}+1\right)d_{\mathrm{BE}}(\cF, D_\Delta, \sqrt{\alpha})\le \left(\frac{\beta}{\alpha}+1\right)d_{\mathrm{BE}}(\cF, D_\Delta, \sqrt{\omega}).
\end{align*}
Now denote $d = d_{\mathrm{BE}}(\cF, D_\Delta, \sqrt{\omega})$, we can get $\alpha\le \frac{d\beta}{t-d}$. Now, $e_t<\alpha\le \frac{d\beta}{t-d}$. Also, recall that $e_t\le R^2$, we can get
\begin{align*}
    \sum_{i=1}^k e_i\II\{e_i>\omega\} \le \min\{d,k\}R^2 + \sum_{i=d+1}^k\left(\frac{d\beta}{t-d}\right) \le \min\{d,k\}R^2 + 2d\beta \log K = \cO(d\beta\log K),
\end{align*}
where the last equality derived from the condition $\beta \ge R^2.$
By choosing $\omega = 1/K$, the equation Eq.~\eqref{eq:BEdiml2} holds.
\end{proof}

\subsection{Proof of Lemma~\ref{lemma:linear mixture l2}}\label{sec:proof of linear mixture}
\begin{proof}

In this subsection, we prove that linear mixture MDP belongs to $\ell_2$-type EC class with $\cT(f) = f^*$ for any $f \in \cF$ and loss function \begin{align*}
    &\ell_{h,f'}(s_{h+1}, \{s_h,a_h\}, f, g) \\&\quad= \theta_{h,g}^T\left[\psi(s_h,a_h)+ \sum_{s'}\phi(s_h, a_h, s')V_{h+1, f'}(s')\right]-r_h-V_{h+1,f'}(s_{h+1}).
\end{align*}
 It is easy to show that the loss function above is bounded. The expectation of the loss function can be calculated by 
\begin{align*}
    \EE_{s_{h+1}}\Big[\ell_{h,f'}(s_{h+1}, \{s_h,a_h\}, f, g)\Big] = (\theta_{h,g}-\theta_h^*)^T \left[\psi(s_h,a_h)+ \sum_{s'}\phi(s_h, a_h, s')V_{h+1,f'}(s')\right].
\end{align*}
Now we prove the loss function satisfies the dominance, decomposable property and $\ell_2$-type condition.
\paragraph{1. Dominance}
\begin{align*}
    &\sum_{i=1}^k (V_{1,f^i}(s_1) - V_{\pi^i}(s_1))\\&\quad\le \sum_{h=1}^H \sum_{i=1}^k \EE_{s_h, a_h \sim \pi^i}[Q_{h,f^i}(s_h, a_h) - r_h-V_{h+1,f^i}(s_{h+1})]\\
    &\quad= \sum_{h=1}^H \sum_{i=1}^k \EE_{s_h, a_h \sim \pi^i}\Big[(\theta_{h,f^i}-\theta_h^*)^T\Big[\psi(s_h, a_h) + \sum_{s'}\phi(s_h,a_h,s')V_{h+1, f^i}(s')\Big]\Big]\\&\quad=
    \sum_{h=1}^H \sum_{i=1}^k\EE_{s_h, a_h \sim \pi^i,s_{h+1}}\Big[\ell_{h,f^i}(s_{h+1}, \{s_h,a_h\}, f^i, f^i)\Big].
\end{align*}
\paragraph{2. Decomposable Property}
\begin{align*}
    &\ell_{h,f'}(s_{h+1}, \{s_h,a_h\}, f, g) - \EE_{s_{h+1}}\Big[\ell_{h,f'}(s_{h+1}, \{s_h,a_h\}, f, g)\Big] \\&\quad= (\theta_h^*)^T\left[\psi(s_h,a_h)+ \sum_{s'}\phi(s_h, a_h, s')V_{h+1}(s')\right]\\
    &\quad= \ell_{h,f'}(s_{h+1}, \{s_h,a_h\}, f, f^*).
\end{align*}
\paragraph{3. $\ell_2$-type Eluder Condition}
First, for any $h$ and $\eta_h =\{s_h, a_h\}$, we have 

\begin{align*}
    &\sum_{i=1}^{k-1}\Big\Vert\EE_{s_{h+1}}\Big[\ell_{h,f^i}(s_{h+1}, \eta_h^i, f^k, f^k)\Big]\Big\Vert_2^2 
    \\&\quad=\sum_{i=1}^{k-1}\left((\theta_{h,f^k}-\theta_h^*)^T \left[\psi(s_h^i,a_h^i)+ \sum_{s'}\phi(s_h^i, a_h^i, s')V_{h+1, f^i}(s')\right]\right)^2.
\end{align*}
Denote $\psi(s_h^i,a_h^i)+ \sum_{s'}\phi(s_h^i, a_h^i, s')V_{h+1, f^k}(s') = W_{h,f^i}(s_h^i, a_h^i)$, $(\theta_{h,f^i}-\theta_h^*) = X_h(f^i)$ 
 and $$\Sigma_k =  I + \sum_{i=1}^{k-1}W_{h,f^i}(s_h^i, a_h^i)W_{h,f^i}(s_h^i, a_h^i)^T,$$ then 
\begin{align*}
    \Vert \theta_{h,f^k}-\theta_{h}^* \Vert_{\Sigma_i}^2 = \Vert X_h(f^k)\Vert_{\Sigma_i}^2 &= \sum_{i=1}^{k-1}\Big\Vert\EE_{s_{h+1}}\Big[\ell_{h,f^i}(s_{h+1}, \eta_h^i, f^k, f^k)\Big]\Big\Vert_2^2  + 4\\
    &\le \beta + 4,
\end{align*}
where $\Vert \theta\Vert_2\le 1$. Now note that $\Vert\psi(s,a)\Vert_2\le 1$ and $\Vert\sum_{s'}\phi(s,a,s')V_{h+1,f}(s')\Vert_2 \le 1$, we can get 
\begin{align*}
    &\sum_{i=1}^k \Big\Vert \EE_{s_{h+1}}\Big[\ell_{h,f^i}(s_{h+1}, \eta_h^i, f^i, f^i)\Big]\Big\Vert_2^2 \\&\quad= \sum_{i=1}^k \left((\theta_{h,f^i}-\theta_h^*)^T \left[\psi(s_h^i,a_h^i)+ \sum_{s'}\phi(s_h^i, a_h^i, s')V_{h+1, f^i}(s')\right]\right)^2\\
    &\quad= \sum_{i=1}^k 4\wedge\left((\theta_{h,f^i}-\theta_h^*)^T \left[\psi(s_h^i,a_h^i)+ \sum_{s'}\phi(s_h^i, a_h^i, s')V_{h+1, f^i}(s')\right]\right)^2\\
    &\quad\le \sum_{i=1}^k 4\wedge \Vert X_h(f^i)\Vert_{\Sigma_i}^2 \Vert W_{h,f^i}(s_h^i, a_h^i)\Vert_{\Sigma_i^{-1}}^2\\
    &\quad\le \sum_{i=1}^k 4\wedge (\beta + 4)\Vert W_{h,f^i}(s_h^i, a_h^i)\Vert_{\Sigma_i^{-1}}^2\\
    &\quad\le (\beta + 4)\sum_{i=1}^k \left(1 \wedge \Vert W_{h,f^i}(s_h^i, a_h^i)\Vert_{\Sigma_i^{-1}}^2\right).
\end{align*}
By the Elliptical Potential Lemma \citep{dani2008stochastic,abbasi2011improved}, 
\begin{align*}
    \sum_{i=1}^k \left(1 \wedge \Vert W_{h,f^i}(s_h^i, a_h^i)\Vert_{\Sigma_i^{-1}}^2\right)&\le \sum_{i=1}^k 2\log(1+\Vert W_{h,f^i}(s_h^i, a_h^i)\Vert_{\Sigma_i^{-1}}^2)\\
    & \le 2\log \frac{\det (\Sigma_{k+1})}{\det (\sigma_0)}.
\end{align*}
Now note that $\det(\Sigma_{k+1})\le \left(\frac{\tr(\Sigma_{k+1})}{d}\right)^{d}$, then 
\begin{align*}
    2\log \frac{\det \Sigma_{k+1}}{\det \Sigma_0} &= 2\log \det(\Sigma_{k+1})-2\log \det(\Sigma_0)\\
    &\le 2d\log\left(1+ \frac{ \sum_{i=1}^k \tr(W_{h,f^i}(s_h^i, a_h^i)W_{h,f^i}(s_h^i, a_h^i)^T)}{d}\right)\\
    & \le 2d\log\left(1 + \frac{\sum_{i=1}^k\Vert W_{h,f^i}(s_h^i, a_h^i)\Vert_2^2}{d}\right)\\
    &\le 2d\log\left(1 + \frac{4k}{d}\right).
\end{align*}
So we can get
\begin{align*}
    \sum_{i=1}^k \Big\Vert \EE_{s_{h+1}}\Big[\ell_{h,f^i}(s_{h+1}, \eta_h^i, f^i, f^i)\Big]\Big\Vert_2^2 \le 2d(\beta + 4)\log \left(1 + \frac{4k}{d}\right) = \cO(d\beta \log k),
\end{align*}
where we ignore all the terms that are independent with $k$.
\end{proof}
\subsection{Proof of Lemma~\ref{lemma:decouped markov l2 type}}
In this subsection, we prove that decoupled Markov Games belong to $\ell_2$-type EC class with $\cT_h Q_{h+1}(s,a,b) = r_h(s,a,b)+\EE_{s'\mid \PP_h(s'\mid s,a,b)}\max_{\upsilon}\min_{\mu}Q_h(s',\upsilon,\mu)$.

\paragraph{1. Dominance}
With probability at least $1-\delta$,
\begin{align}
    &\sum_{i=1}^k (V_{1,f^i}(s_1)-V_{\pi^i}(s_1))\nonumber\\&\quad= \sum_{h=1}^H \sum_{i=1}^k \EE_{s_{h+1}}\EE_{s_h\sim \pi^i}\Big[V_{h,f^i}(s_h)-r_h-V_{h+1,f^i}(s_{h+1})\Big]\nonumber\\
    &\quad= \sum_{h=1}^H \sum_{i=1}^k \EE_{s_{h+1}}\EE_{\pi^i}\Big[\min_{\mu}\PP_{\upsilon^i, \mu}Q_{h,f^i}(s_h, \upsilon^i, \mu)-r_h-V_{h+1,f^i}(s_{h+1})\Big]\label{eq:best response}\\
    &\quad\le \sum_{h=1}^H \sum_{i=1}^k \EE_{s_{h+1}}\EE_{ \pi^i}\Big[\PP_{\upsilon^i, \mu^i}Q_{h,f^i}(s_h, \upsilon^i, \mu^i)-r_h-V_{h+1,f^i}(s_{h+1})\Big]\nonumber\\
    &\quad=\sum_{h=1}^H \sum_{i=1}^k \EE_{s_{h+1}}\EE_{\pi^i}\Big[Q(s_h,a_h,b_h)-r_h-V_{h+1,f^i}(s_{h+1})\Big]\label{eq: expectation ab}\\
    &\quad=\sum_{h=1}^H \sum_{i=1}^k \EE_{s_{h+1}}\Big[Q_{h,f^i}(s_h^i,a_h^i,b_h^i)-r_h-V_{h+1,f^i}(s_{h+1})\Big] + \cO(\sqrt{KH}\log (KH/\delta))\label{ineq:azuma2}\\
    &\quad= \sum_{h=1}^H \sum_{i=1}^k \EE_{s_{h+1}}\Big[\ell_{h,f^i}(s_{h+1}, \{s_h^i,a_h^i,b_h^i\}, f^i, f^i)\Big] + \cO(\sqrt{KH}\log (KH/\delta))\nonumber,
\end{align}
where the Eq.\eqref{eq:best response} holds because the greedy policy $\upsilon^i = \upsilon_{f^i}$ satisfies that $$V_{h,f^i}(s_h) = \min_\mu \PP_{\upsilon^i,\mu}Q_{h,f^i}(s_h,\upsilon^i,\mu).$$ Eq.\eqref{eq: expectation ab} holds because $\pi_h^i = (\upsilon^i, \mu^i)$, and Eq.\eqref{ineq:azuma2} follows from Azuma-Hoeffding's inequality. 
\paragraph{2. Decomposable Property}

\begin{align*}
    &\ell_{h,f'}(\zeta_h, \eta_h,f,g) - \EE_{\zeta_h}\Big[\ell_{h,f'}(\zeta_h, \eta_h,f,g)\Big]\\ &\quad= \Big[Q_{h,g}(s_h,a_h, b_h) - r(s_h,a_h, b_h) - V_{h+1,f}(s_{h+1})\Big]\\&\qquad\quad-\Big[Q_{h,g}(s_h,a_h, b_h) - (\cT_hV_{h+1,f})(s_h,a_h, b_h)\Big]\\
    &\quad=\Big[(\cT_hV_{h+1,f})(s_h,a_h, b_h)-r(s_h,a_h, b_h) - V_{h+1,f}(s_{h+1})\Big]\\
    &\quad=\ell_{h,f'}(\zeta_h, \eta_h, f, \cT(f)).
\end{align*}

\paragraph{3. $\ell_2$-type eluder Condition}
Note that $\EE_{\zeta_h}[\ell_{h,f'}(\zeta_h,\eta_h,f,g)]=[f_h - \cT_h f_{h+1}](s_h^i, a_h^i,b_h^i)$ is the Bellman residual of Markov Games. The proof can be derived similarly to  Lemma~\ref{lemma:BEdim l2} by replacing $\{a_h\}$ and the Bellman operator for single-agent MDP to $\{a_h,b_h\}$ and the Bellman operator for the two-player zero-sum MG.

\subsection{Proof of Lemma~\ref{lemma:KNRl2}}
\begin{proof}
We proof this Lemma by the classical $\ell_2$ eluder argument. First, denote $U_{h,f,j}$ with $j \in [d_s]$ as the $j$-th row of $U_{h,f}$, then 
\begin{align*}
    \Vert(U_{h,f}-U^*_h)\phi(s,a)\Vert_2^2 = \sum_{j=1}^{d_s}\Vert(U_{h,f,j}-U_{h,j}^*)\phi(s,a)\Vert_2^2.
\end{align*}
Then, denote $\Sigma_k =  \sum_{i=1}^{k-1} \phi(s_h^i, a_h^i)\phi(s_h^i, a_h^i)^T + \lambda I $, we can get
\begin{align*}
    \sum_{j=1}^{d_s}\Vert U_{h,f^k,j}-U_{h,j}^*\Vert_{\Sigma_k}^2 &= \sum_{i=1}^{k-1}\sum_{j=1}^{d_s}((U_{h,f^k,j}-U_{h,j}^*)\phi(s_h^i, a_h^i))^2 + \lambda \sum_{j=1}^{d_s}\Vert U_{h,f^k,j}-U_{h,j}^*\Vert_2^2\\
    &\le \beta + \lambda \cdot d_sR^2,
\end{align*}
where $\Vert U_{h,f,j}\Vert_2^2\le \Vert U_{h,f}\Vert_2^2\le R^2$ for any $f \in \cF$.
Now, recall that $\Vert \phi(s,a)\Vert_2\le 1$, then $\Vert (U_{h,f^i}-U_{h}^*)\phi(s,a)\Vert_2^2\le 4R^2$. Hence, choosing $\lambda = \frac{4}{d_s}$,
\begin{align*}
    \sum_{i=1}^k \Vert (U_{h,f^i}-U_{h}^*)\phi(s_h^i, a_h^i)\Vert_2^2 &=\sum_{i=1}^k \left(\Vert (U_{h,f^i}-U_{h}^*)\phi(s_h^i, a_h^i)\Vert_2^2\wedge 4R^2\right) \\
    & =  \sum_{i=1}^k\left(\left(\sum_{j=1}^{d_s} \Vert (U_{h,f^i,j}-U_{h,j}^*)\phi(s_h^i,a_h^i)\Vert_2^2\right) \wedge 4R^2 \right)\\
    &\le \sum_{i=1}^k\left(\left(\sum_{j=1}^{d_s} \Vert (U_{h,f^i,j}-U_{h,j}^*)\Vert_{\Sigma_i}^2 \Vert \phi(s_h^i, a_h^i)\Vert_{\Sigma_i^{-1}}^2\right) \wedge 4R^2 \right)\\
    &\le \sum_{i=1}^k \left(\left(\Vert \phi(s_h^i, a_h^i)\Vert_{\Sigma_i^{-1}}^2\sum_{j=1}^{d_s} \Vert U_{h,f^k,j}-U_{h,j}^*\Vert_{\Sigma_k}^2 \right) \wedge 4R^2 \right)\\
    &\le \sum_{i=1}^k \left(\left(\Vert \phi(s_h^i, a_h^i)\Vert_{\Sigma_i^{-1}}^2(\beta + 4R^2) \right) \wedge 4R^2 \right)\\
    &\le \sum_{i=1}^k (\beta + 4R^2)\left(1\wedge \Vert \phi(s_h^i, a_h^i)\Vert_{\Sigma_i^{-1}}^2\right).
\end{align*}
By the Elliptical Potential Lemma \citep{dani2008stochastic,abbasi2011improved}, we have 
\begin{align*}
    \sum_{i=1}^k\left(1\wedge \Vert \phi(s_h^i, a_h^i)\Vert_{\Sigma_i^{-1}}^2\right) & \le \sum_{i=1}^k 2\log (1+\Vert \phi(s_h^i, a_h^i)\Vert_{\Sigma_i^{-1}}^2)\\
    &\le 2\log \frac{\det \Sigma_{k+1}}{\det \Sigma_0}.
\end{align*}
Now note that $\det(\Sigma_{k+1}) \le \left(\frac{\tr(\Sigma_{k+1})}{d_\phi}\right)^{d_\phi}$, then 
\begin{align*}
    2\log \frac{\det \Sigma_{k+1}}{\det \Sigma_0} &= 2\log \det(\Sigma_{k+1})-2\log \det(\Sigma_0)\\
    &\le 2d_\phi\log\left(\lambda + \frac{ \sum_{i=1}^k \tr(\phi(s_h^i, a_h^i)\phi(s_h^i, a_h^i)^T)}{d_\phi}\right)\\
    & \le 2d_\phi\log\left(\lambda + \frac{\sum_{i=1}^k\Vert \phi(s_h^i, a_h^i)\Vert_2^2}{d_\phi}\right)\\
    &\le 2d_\phi\log\left(\lambda + \frac{k}{d_\phi}\right).
\end{align*}
Thus 
\begin{align*}
    \sum_{i=1}^k \Vert (U_{h,f^i}-U_{h}^*)\phi(s_h^i, a_h^i)\Vert_2^2 \le 2d_\phi(\beta + 4R^2)\log \left(\frac{4}{d_s}+\frac{k}{d_\phi}\right) =\cO(d_\phi \beta \log k),
\end{align*}
where we ignore all terms independent with $k$, and regard $R$  as a  constant.
\end{proof}


}

\end{document}